\documentclass{article}

\usepackage{graphicx,amsthm,amsmath,amssymb,bm,color,fullpage,cite,caption,comment,algpseudocode,titlesec,algorithm,subcaption}

\usepackage[utf8]{inputenc} 
\usepackage[T1]{fontenc}    
\usepackage{url}            
\usepackage{booktabs}       
\usepackage{amsfonts}       
\usepackage{nicefrac}       
\usepackage{microtype}      

  \makeatletter
\def\BState{\State\hskip-\ALG@thistlm}
\makeatother

  \usepackage{mathtools}

\titleformat{\paragraph}
{\normalfont\normalsize\bfseries}{\theparagraph}{1em}{}
\titlespacing*{\paragraph}
{0pt}{3.25ex plus 1ex minus .2ex}{1.5ex plus .2ex}

\usepackage{movie15}

\usepackage[bottom,hang,flushmargin]{footmisc} 

\usepackage{hyperref}
\definecolor{darkred}{RGB}{150,0,0}
\definecolor{darkgreen}{RGB}{0,150,0}
\definecolor{darkblue}{RGB}{0,0,200}
\hypersetup{colorlinks=true, linkcolor=darkred, citecolor=darkgreen, urlcolor=darkblue}

\newtheorem{theorem}{Theorem}[section]

\newtheorem{assumption}{Assumption}

\newtheorem{lemma}[theorem]{Lemma}
\newtheorem{corollary}[theorem]{Corollary}

\newtheorem{definition}[theorem]{Definition}

\newcommand{\tr}[1]{{\text{tr}(#1)}}
\newcommand{\eps}{\varepsilon}

\newcommand{\dquad}{\quad,\quad}
\newcommand{\distas}{\overset{\text{i.i.d.}}{\sim}}

\newcommand{\beq}{\begin{equation}}
\newcommand{\eeq}{\end{equation}}

\newcommand{\var}[1]{{{\text{\bf{var}}}}[#1]}
\newcommand{\zm}[1]{{{\text{\bf{zm}}}}(#1)}

\newcommand{\nn}{\nonumber}
\newcommand{\la}{\lambda}
\newcommand{\A}{{\mtx{A}}}
\newcommand{\Ah}{{\mtx{\hat{A}}}}

\newcommand{\Bh}{{\mtx{\hat{B}}}}

\newcommand{\smn}{s_{\min}}
\newcommand{\lmn}[1]{s_{\min}(#1)}

\newcommand{\Ub}{{\mtx{U}}}
\newcommand{\NK}{{\bar{N}}}

\newcommand{\B}{{{\mtx{B}}}}

\newcommand{\Gb}{{\mtx{G}}}

\newcommand{\Lc}{{\cal{L}}}

\newcommand{\Aba}{{\bar{\A}}}

\newcommand{\TK}{{L}}

\newcommand{\Pb}{{\mtx{P}}}

\newcommand{\Qb}{{\mtx{Q}}}
\newcommand{\Qbt}{{\mtx{\tilde{Q}}}}
\newcommand{\Cb}{{\mtx{C}}}
\newcommand{\Hb}{{\mtx{H}}}

\newcommand{\Hc}{{\mtx{\tilde{H}}}}
\newcommand{\Hbb}{{\mtx{\bar{H}}}}
\newcommand{\Uc}{\mtx{\tilde{U}}}

\newcommand{\bSi}{{\boldsymbol{{\Sigma}}}}
\newcommand{\bSio}[1]{{\boldsymbol{{\Sigma}}}[#1]}

\newcommand{\Db}{{\mtx{D}}}

\newcommand{\onebb}{{\mathbf{1}}}
\newcommand{\Iden}{{\mtx{I}}}
\newcommand{\M}{{\mtx{M}}}

\newcommand{\order}[1]{{\cal{O}}(#1)}

\newcommand{\el}{{\ell}}
\newcommand{\tn}[1]{\|{#1}\|_{\ell_2}}
\newcommand{\tf}[1]{\|{#1}\|_{F}}

\newcommand{\te}[1]{\|{#1}\|_{\psi_1}}
\newcommand{\tsub}[1]{\|{#1}\|_{\psi_2}}

\newcommand{\Cc}{\mathcal{C}}

\newcommand{\bbeta}{{\boldsymbol{\beta}}}
\newcommand{\bteta}{{\boldsymbol{\theta}}}
\newcommand{\bTeta}{{\boldsymbol{\Theta}}}

\newcommand{\Bc}{\mathcal{B}}
\newcommand{\Sc}{\mathcal{S}}

\newcommand{\Nn}{\mathcal{N}}

\newcommand{\vb}{\vct{v}}

\newcommand{\abb}{\mtx{\bar{a}}}

\newcommand{\w}{\vct{w}}

\newcommand{\li}{\left<}
\newcommand{\ri}{\right>}

\newcommand{\ab}{\vct{a}}

\newcommand{\ub}{{\vct{u}}}
\newcommand{\cb}{{\vct{c}}}

\newcommand{\h}{\vct{h}}

\newcommand{\g}{{\vct{g}}}

\newcommand{\tao}{\bar{\tau}}
\newcommand{\Zb}{\mtx{Z}}

\newcommand{\xh}{\hat{\x}}

\newcommand{\yh}{\hat{\y}}

\newcommand{\uu}[1]{\ub^{(#1)}}

\newcommand{\hi}[1]{\h^{(#1)}}
\newcommand{\hb}[1]{\bar{\h}^{(#1)}}

\newcommand{\hh}{\h}
\newcommand{\hbr}{\bar{\h}}

\newcommand{\lip}{B}

\newcommand{\x}{\vct{x}}

\newcommand{\y}{\vct{y}}

\newcommand{\W}{\mtx{W}}
\newcommand{\bgl}{{~\big |~}}

\definecolor{emmanuel}{RGB}{255,127,0}

\newcommand{\qb}{{\vct{q}}}
\newcommand{\R}{\mathbb{R}}
\newcommand{\Pro}{\mathbb{P}}

\newcommand{\E}{\operatorname{\mathbb{E}}}
\newcommand{\Eb}{\mtx{E}}

\newcommand{\grad}[1]{{\nabla\Lc(#1)}}

\newcommand{\vct}[1]{\bm{#1}}
\newcommand{\mtx}[1]{\bm{#1}}

\newcommand{\X}{{\mtx{X}}}
\newcommand{\Xc}{{\mtx{\tilde{X}}}}
\newcommand{\Y}{{\mtx{Y}}}

\numberwithin{equation}{section}

\makeatletter
\def\BState{\State\hskip-\ALG@thistlm}
\makeatother

\title{Fast Convergence Guarantees for Learning\\Simple Recurrent Neural Networks}
\title{Fast Convergence of Gradient Descent for\\Learning Simple Recurrent Neural Networks}
\title{Gradient Descent Learns $X$ where $X=\text{Recurrent Neural Networks}$}
\title{Gradient Descent Learns Simple Recurrent Neural Networks}
\title{Stochastic Gradient Descent Learns Nonlinear State Equations}
\title{Stochastic Gradient Descent Learns \\State Equations with Nonlinear Activations}

\author{Samet Oymak\\ {\small{University of California, Riverside}}}
\date{}
\begin{document}
\maketitle
\begin{abstract} We study discrete time dynamical systems governed by the state equation $\h_{t+1}=\phi(\A\h_t+\B\ub_t)$. Here $\A,\B$ are weight matrices, $\phi$ is an activation function, and $\ub_t$ is the input data. This relation is the backbone of recurrent neural networks (e.g.~LSTMs) which have broad applications in sequential learning tasks. We utilize stochastic gradient descent to learn the weight matrices from a finite input/state trajectory $\{\ub_t,\h_t\}_{t=0}^N$. We prove that SGD estimate linearly converges to the ground truth weights while using near-optimal sample size. Our results apply to increasing activations whose derivatives are bounded away from zero. The analysis is based on i) a novel SGD convergence result with nonlinear activations and ii) careful statistical characterization of the state vector. Numerical experiments verify the fast convergence of SGD on ReLU and leaky ReLU in consistence with our theory.
\end{abstract}
\section{Introduction}
A wide range of problems involve sequential data with a natural temporal ordering. Examples include natural language processing, time series prediction, system identification, and control design, among others. State-of-the-art algorithms for sequential problems often stem from dynamical systems theory and are tailored to learn from temporally dependent data.  Linear models and algorithms; such as Kalman filter, PID controller, and linear dynamical systems, have a long history and are utilized in control theory since 1960's with great success \cite{brown1992introduction,ho1966effective,aastrom1995pid}. More recently, nonlinear models such as recurrent neural networks (RNN) found applications in complex tasks such as machine translation and speech recognition \cite{bahdanau2014neural,graves2013speech,hochreiter1997long}. Unlike feedforward neural networks, RNNs are dynamical systems that use their internal state to process inputs. The goal of this work is to shed light on the inner workings of RNNs from a theoretical point of view. In particular, we focus on the RNN state equation which is characterized by a nonlinear activation function $\phi$, state weight matrix $\A$, and input weight matrix $\B$ as follows
\begin{align}
\h_{t+1}=\phi(\A\h_t+\B\ub_t),\label{main rel}
\end{align}
Here $\h_t$ is the state vector and $\ub_t$ is the input data at timestamp $t$. This equation is the source of dynamic behavior of RNNs and distinguishes RNN from feedforward networks. The weight matrices $\A$ and $\B$ govern the dynamics of the state equation and are inferred from data. We will explore the statistical and computational efficiency of stochastic gradient descent (SGD) for learning these weight matrices.


\noindent {\bf{Contributions:}} Suppose we are given a finite trajectory of input/state pairs $(\ub_t,\h_t)_{t=0}^N$ generated from the state equation \eqref{main rel}. We consider a least-squares regression obtained from $N$ equations; with inputs $(\ub_t,\h_t)_{t=1}^N$ and outputs $(\h_{t+1})_{t=1}^N$. For a class of activation functions including leaky ReLU and for stable systems\footnote{Throughout this work, a system is called stable if the spectral norm of the state matrix $\A$ is less than $1$.}, we show that SGD {\em{linearly converges}} to the ground truth weight matrices while requiring near-optimal trajectory length $N$. In particular, the required sample size is $\order{n+p}$ where $n$ and $p$ are the dimensions of the state and input vectors respectively. Our results are extended to unstable systems when the samples are collected from multiple independent RNN trajectories rather than a single trajectory. Our results apply to increasing activation functions whose derivatives are bounded away from zero; which includes leaky ReLU. Numerical experiments on ReLU and leaky ReLU corroborate our theory and demonstrate that SGD converges faster as the activation slope increases. To obtain our results, we i) characterize the statistical properties of the state vector (e.g.~well-conditioned covariance) and ii) derive a novel SGD convergence result with nonlinear activations; which may be of independent interest. As a whole, this paper provides a step towards foundational understanding of RNN training via SGD.






\subsection{Related Work}
Our work is related to the recent theory literature on linear dynamical systems (LDS) and neural networks. 
\noindent {\bf{Linear dynamical systems:}} The state-equation \eqref{main rel} reduces to a LDS when $\phi$ is the linear activation ($\phi(x)=x$). Identifying the weight matrices is a core problem in linear system identification and is related to the optimal control problem (e.g.~linear quadratic regulator) with unknown system dynamics. While these problems are studied since 1950's \cite{ljung1998system,ljung1987system,aastrom1971system}, our work is closer to the recent literature that provides data dependent bounds and characterize the non-asymptotic learning performance. Recht and coauthors \cite{simchowitz2018learning,tu2018approximation,tu2017non,hardt2016gradient} have a series of papers exploring optimal control problem. In particular, Hardt et al. shows gradient descent learns single-input-single-output (SISO) LDS with polynomial guarantees \cite{hardt2016gradient}. Oymak and Ozay provides guarantees for learning multi-input-multi-output (MIMO) LDS \cite{oymak2018non}. Sanandaji \cite{sanandaji2011compressive,sanandaji2011exact} et al. studies the identification of sparse systems.


\noindent {\bf{Neural networks:}} There is a growing literature on the theoretical aspects of deep learning and provable algorithms for training neural networks. Most of the existing results are concerned with feedforward networks \cite{sol2017,zhong2017recovery,brutzkus2017globally, soltanolkotabi2017learning, oymak2018learning,zhong2017learning,janzamin2015beating,li2017convergence,mei2018mean}. \cite{li2017convergence,mei2018mean,janzamin2015beating,sol2017} consider learning fully-connected shallow networks with gradient descent. \cite{brutzkus2017globally,zhong2017learning,oymak2018end,du2017convolutional} address convolutional neural networks; which is an efficient weight-sharing architecture. \cite{brutzkus2017sgd,wang2018learning} studies over-parameterized networks when data is linearly separable. \cite{oymak2018end,janzamin2015beating} utilize tensor decomposition techniques for learning feedforward neural nets. For recurrent networks, Sedghi and Anandkumar \cite{sedghi2016training} proposed tensor algorithms with polynomial guarantees and Khrulkov et al.~\cite{khrulkov2017expressive} studied their expressive power. More recently, Miller and Hardt \cite{miller2018recurrent} showed that stable RNNs can be approximated by feed-forward networks. 

\section{Problem Setup}\label{setup sec}
We first introduce the notation. $\|\cdot\|$ returns the spectral norm of a matrix and $\lmn{\cdot}$ returns the minimum singular value. The activation $\phi:\R\rightarrow\R$ applies entry-wise if its input is a vector. Throughout, $\phi$ is assumed to be a $1$-Lipschitz function. With proper scaling of its parameters, the system \eqref{main rel} with a Lipschitz activation can be transformed into a system with $1$-Lipschitz activation. The functions $\bSio{\cdot}$ and $\var{\cdot}$ return the covariance of a random vector and variance of a random variable respectively. $\Iden_n$ is the identity matrix of size $n\times n$. Normal distribution with mean ${\boldsymbol{\mu}}$ and covariance $\bSi$ is denoted by $\Nn({\boldsymbol{\mu}},\bSi)$. Throughout, $c,C,c_0,c_1,\dots$ denote positive absolute constants.
\vspace{5pt}

\noindent{\bf{Setup:}} We consider the dynamical system parametrized by an activation function $\phi(\cdot)$ and weight matrices $\A\in\R^{n\times n},\B\in\R^{n\times p}$ as described in \eqref{main rel}. Here, $\hh_t$ is the $n$ dimensional state-vector and $\ub_t$ is the $p$ dimensional input to the system at time $t$. As mentioned previously, \eqref{main rel} corresponds to the state equation of a recurrent neural network. For most RNNs of interest, the state $\hh_t$ is hidden and we only get to interact with $\hh_t$ via an additional output equation. For Elman networks \cite{elman1990finding}, this equation is characterized by some output activation $\phi_y$ and output weights $\Cb,\Db$ as follows
\begin{align}
\y_t=\phi_y(\Cb\hh_t+\Db\ub_t).\label{out observe}
\end{align} In this work, our attention is restricted to the state equation \eqref{main rel}; which corresponds to setting $\y_t=\h_t$ in the output equation. To analyze \eqref{main rel} in a non-asymptotic data-dependent setup, we assume a finite input/state trajectory of $\{\ub_t,\hh_{t}\}_{t=0}^{N}$ generated by some ground truth weight matrices $(\A,\B)$. Our goal is learning the unknown weights $\A$ and $\B$ in a data and computationally efficient way. In essence, we will show that, if the trajectory length satisfies $N\gtrsim n+p$, SGD can quickly and provably accomplish this goal using a constant step size.





\begin{algorithm} [!t]\caption{Learning state equations with nonlinear activations}\label{algo 1}
\begin{algorithmic}[1]
\item {\bf{Inputs:}} $(\y_t,\h_{t},\ub_{t})_{t=1}^N$ sampled from a trajectory. Scaling $\mu$, learning rate $\eta$. Initialization $\A_0,\B_0$.
\item {\bf{Outputs:}} Estimates $\Ah,\Bh$ of the weight matrices $\A,\B$.
\item $\x_t\gets [{\mu}\h_t^T~\ub_t^T]^T$ for $1\leq t\leq N$.
\item $\bTeta_0\gets[\mu^{-1}\A_0~\B_0]$
{\For {$\tau$ from $1$ to \text{END}}
\State Pick $\gamma_\tau$ from $\{1,2,\dots,N\}$ uniformly at random.
\State $\bTeta_{\tau}\gets \bTeta_{\tau-1}-\eta \nabla \Lc_{\gamma_\tau}(\bTeta_{\tau-1})$
\EndFor}\\
\Return $[\Ah~\Bh]\gets\bTeta_{\text{END}}\begin{bmatrix}\mu\Iden_n&0\\0&\Iden_p\end{bmatrix}$.
\end{algorithmic}
\end{algorithm}

\noindent {\bf{Appoach:}} Our approach is described in Algorithm \ref{algo 1}. It takes two hyperparameters; the scaling factor $\mu$ and learning rate $\eta$. Using the RNN trajectory, we construct $N$ triples of the form $\{\ub_t,\hh_t,\hh_{t+1}\}_{t=1}^{N}$. We formulate a regression problem by defining the output vector $\y_t$, input vector $\x_t$, and the target parameter $\Cb$ as follows
\begin{align}
\y_t=\h_{t+1}\quad,\quad\x_t=\begin{bmatrix}\mu\h_t\\\ub_t\end{bmatrix}\in\R^{n+p}\quad,\quad\Cb=[\mu^{-1}\A~\B]\in\R^{n\times (n+p)}.\label{reparam}
\end{align} 
With this reparameterization, we find the input/output identity $\y_{t}=\phi(\Cb\x_t)$. We will consider the least-squares regression given by
\begin{align}
\Lc(\bTeta)=\frac{1}{N}\sum_{t=1}^N \Lc_t(\bTeta)\quad\text{where}\quad \Lc_t(\bTeta)=\frac{1}{2}\tn{\y_{t}-\phi(\bTeta\x_t)}^2.\label{sgd loss}
\end{align}
For learning the ground truth parameter $\Cb$, we utilize SGD on the loss function \eqref{sgd loss} with a constant learning rate $\eta$. Starting from an initial point $\bTeta_0$, after $\text{END}$ SGD iterations, Algrorithm \ref{algo 1} returns an estimate $\hat{\Cb}=\Theta_{\text{END}}$. Estimates of $\A$ and $\B$ are decoded from the left and right submatrices of $\hat{\Cb}$ respectively.

\section{Main Results}\label{main result sec}
\subsection{Preliminaries}
The analysis of the state equation naturally depends on the choice of the activation function; which is the source of nonlinearity. We first define a class of Lipschitz and increasing activation functions.
\begin{definition}[$\beta$-increasing activation]\label{betain} Given $1\geq \beta\geq  0$, the activation function $\phi$ satisfies $\phi(0)=0$ and $1\geq \phi'(x)\geq\beta$ for all $x\in\R$.
\end{definition}
Our results will apply to strictly increasing activations where $\phi$ is $\beta$-increasing for some $\beta>0$. Observe that, this excludes ReLU activation which has zero derivative for negative values. However, it includes Leaky ReLU which is a generalization of ReLU. Parameterized by $1\geq \beta\geq 0$, Leaky ReLU is a $\beta$-increasing function given by
\begin{align}
\text{LReLU}(x)=\max(\beta x,x).\label{leaky relu}
\end{align}
In general, given an increasing and $1$-Lipschitz activation $\phi$, a $\beta$-increasing function $\phi_\beta$ can be obtained by blending $\phi$ with the linear activation, i.e. $\phi_\beta(x)=(1-\beta)\phi(x)+\beta x$.

A critical property that enables SGD is that the state-vector covariance $\bSio{\h_t}$ is well-conditioned under proper assumptions. The lemma below provides upper and lower bounds on this covariance matrix in terms of problem variables.
\begin{lemma}[State vector covariance] \label{stcov}Consider the state equation \eqref{main rel} where $\h_0=0$ and $\ub_t\distas\Nn(0,\Iden_p)$. Define the upper bound term $B_t$ as
\begin{align}
B_t=\|\B\|\sqrt{\frac{1-\|\A\|^{2t}}{1-\|\A\|^2}}.\label{bteq}
\end{align}
\begin{itemize}
\item Suppose $\phi$ is $1$-Lipschitz and $\phi(0)=0$. Then, for all $t\geq 0$, $\bSio{\h_t}\preceq B_t^2\Iden_n$.
\item Suppose $\phi$ is a $\beta$-increasing function and $p\geq n$. Then, $\bSio{\h_t}\succeq  \beta^2\lmn{\B}^2\Iden_n$.
\end{itemize}
\end{lemma}
As a natural extension from linear dynamical systems, we will say the system is stable if $\|\A\|<1$ and unstable otherwise. For activations we consider, stability implies that if the input is set to $0$, state vector $\h_t$ will exponentially converge to $0$ i.e.~the system forgets the past states quickly. This is also the reason $(B_t)_{t\geq 0}$ sequence converges for stable systems and diverges otherwise. The condition number of the covariance will play a critical role in our analysis. Using Lemma \ref{stcov}, this number can be upper bounded by ${\rho}$ defined as
\begin{align}
\rho=\left(\frac{B_\infty}{\beta\lmn{\B}}\right)^2=\left(\frac{\|\B\|}{\lmn{\B}}\right)^2\frac{1}{\beta^2(1-\|\A\|^2)}.\label{rhodef}
\end{align}
Observe that, the condition number of $\B$ appears inside the $\rho$ term.

\subsection{Learning from Single Trajectory}
Our main result applies to stable systems ($\|\A\|<1$) and provides a non-asymptotic convergence guarantee for SGD in terms of the upper bound on the state vector covariance. This result characterizes the sample complexity and the rate of convergence of SGD; and also provides insights into the role of activation function and the spectral norm of $\A$.
\begin{theorem} [Main result] \label{main thm} Let $\{\ub_t,\h_{t+1}\}_{t=1}^N$ be a finite trajectory generated from the state equation \eqref{main rel}. Suppose $\|\A\|<1$, $\phi$ is $\beta$-increasing, $\h_0=0$, $p\geq n$, and $\ub_t\distas\Nn(0,\Iden_p)$. Let $\rho$ be same as \eqref{rhodef} and $c,C,c_0$ be properly chosen absolute constants. Pick the trajectory length $N$ to satisfy
\[
N\geq CL \rho^2(n+p),
\]
where $ L= 1-\frac{\log(cn\rho)}{\log \|\A\|}$. Pick scaling $\mu=1/B_\infty$, learning rate $\eta=c_0\frac{\beta^2}{\rho n(n+p)}$, and consider the loss function \eqref{sgd loss}. With probability $1-4N\exp(-100n)-8L\exp(-\order{\frac{N}{L\rho^2}})$, starting from an initial point $\bTeta_0$, for all $\tau\geq 0$, the SGD iterations described in Algorithm \ref{algo 1} satisfies
\begin{align}
\E[\tf{\bTeta_{\tau}-\Cb}^2]\leq (1-c_0\frac{\beta^4}{2\rho^2 n(n+p)})^{\tau}\tf{\bTeta_{0}-\Cb}^2.\label{conv bounddd}
\end{align}
Here the expectation is over the randomness of the SGD updates.
\end{theorem}
\noindent{\bf{Sample complexity:}} Theorem \ref{main thm} essentially requires $N\gtrsim (n+p)/{\beta^4}$ samples for learning. This can be seen by unpacking \eqref{rhodef} and ignoring the logarithmic $L$ term and the condition number of $\B$. Observe that $\order{n+p}$ growth achieves near-optimal sample size for our problem. Each state equation \eqref{main rel} consists of $n$ sub-equations (one for each entry of $\h_{t+1}$). We collect $N$ state equations to obtain a system of $Nn$ equations. On the other hand, the total number of unknown parameters in $\A$ and $\B$ are $n(n+p)$. This implies Theorem \ref{main thm} is applicable as soon as the problem is mildly overdetermined i.e.~$Nn\gtrsim n(n+p)$.

\noindent{\bf{Computational complexity:}} Theorem \ref{main thm} requires $\order{n(n+p)\log\frac{1}{\eps}}$ iterations to reach $\eps$-neighborhood of the ground truth. Our analysis reveals that, this rate can be accelerated if the state vector is zero-mean. This happens for odd activation functions satisfying $\phi(-x)=-\phi(x)$ (e.g.~linear activation). The result below is a corollary and requires $\times n$ less iterations. 
\begin{theorem} [Faster learning for odd activations] \label{thm odd}Consider the same setup provided in Theorem \ref{main thm}. Additionally, assume that $\phi$ is an odd function. Pick scaling $\mu=1/B_\infty$, learning rate $\eta=c_0\frac{\beta^2}{\rho (n+p)}$, and consider the loss function \eqref{sgd loss}. With probability $1-4N\exp(-100n)-8L\exp(-\order{\frac{N}{L\rho^2}})$, starting from an initial point $\bTeta_0$, for all $\tau\geq 0$, the SGD iterations described in Algorithm \ref{algo 1} satisfies
\begin{align}
\E[\tf{\bTeta_{\tau}-\Cb}^2]\leq (1-c_0\frac{\beta^4}{2\rho^2 (n+p)})^{\tau}\tf{\bTeta_{0}-\Cb}^2,\label{conv bound2}
\end{align}
where the expectation is over the randomness of the SGD updates.
\end{theorem}
Another aspect of the convergence rate is the dependence on $\beta$. In terms of $\beta$, the SGD error \eqref{conv bounddd} decays as $(1-\order{\beta^8})^{\tau}$. While it is not clear how optimal is the exponent $8$, numerical experiments in Section \ref{numeric sec} demonstrate that larger $\beta$ indeed results in drastically faster convergence.
\section{Main Ideas and Proof Strategy}\label{gen strat}
To prove the results of the previous section, we derive a deterministic result that establishes the linear convergence of SGD for $\beta$-increasing functions. For linear convergence proofs, a typical strategy is showing the {\em{strong convexity}} of the loss function i.e.~showing that, for some $\alpha>0$ and all points $\vb,\ub$, the gradient satisfies
\[
\li\nabla\Lc(\vb)-\nabla\Lc(\ub),\vb-\ub\ri\geq \alpha \tn{\vb-\ub}^2.
\]
The core idea of our convergence result is that the strong convexity parameter of the loss function with $\beta$-increasing activations can be connected to the loss function with {\em{linear activations}}. In particular, recalling \eqref{sgd loss}, set $\y^{\text{lin}}_t=\Cb\x_t$ and define the linear loss to be
\[
\Lc^{\text{lin}}(\bTeta)=\frac{1}{2N}\sum_{i=1}^N \tn{\y^{\text{lin}}_t-\bTeta\x_t}^2.
\]
Denoting the strong convexity parameter of the original loss by $\alpha_\phi$ and that of linear loss by $\alpha_{\text{lin}}$, we argue that $\alpha_{\phi}\geq \beta^2\alpha_{\text{lin}}$; which allows us to establish a convergence result as soon as $\alpha_{\text{lin}}$ is strictly positive. Next result is our SGD convergence theorem which follows from this discussion.
\begin{theorem}[Deterministic convergence]\label{det conv} Suppose a data set $\{\x_i,\y_i\}_{i=1}^N$ is given; where output $\y_i$ is related to input $\x_i$ via $\y_i=\phi(\li\x_i,\bteta\ri)$ for some $\bteta\in\R^n$. Suppose $\beta>0$ and $\phi$ is a $\beta$-increasing. Let $\gamma_+\geq\gamma_->0$ be scalars. Assume that input samples satisfy the bounds
\[
\gamma_+\Iden_n\succeq\frac{1}{N}\sum_{i=1}^N\x_i\x_i^T\succeq\gamma_-\Iden_n\quad,\quad \tn{\x_i}^2\leq B~\text{for all}~i.
\]
Let $\{r_\tau\}_{\tau=0}^\infty$ be a sequence of i.i.d.~integers uniformly distributed between $1$ to $N$. Then, starting from an arbitrary point $\bteta_0$, setting learning rate $\eta=\frac{\beta^2\gamma_-}{\gamma_+B}$, for all $\tau\geq 0$, the SGD iterations for quadratic loss
\begin{align}
\bteta_{\tau+1}=\bteta_\tau-\eta(\phi(\x_{r_\tau}^T\bteta_\tau)-\y_{r_\tau})\phi'(\x_{r_\tau}^T\bteta_\tau)\x_{r_\tau},
\end{align}
satisfies the error bound
\begin{align}\label{adrec}
\E[\tn{\bteta_\tau-\bteta}^2]\leq \tn{\bteta_0-\bteta}^2(1-\frac{\beta^4\gamma_-^2}{\gamma_+B})^{\tau},
\end{align}
where the expectation is over the random selection of the SGD iterations $\{r_\tau\}_{\tau=0}^\infty$.
\end{theorem}
This theorem provides a clean convergence rate for SGD for $\beta$-increasing activations and naturally generalizes standard results on linear regression which corresponds to $\beta=1$. Its extension to proximal gradient methods might be beneficial for high-dimensional nonlinear problems (e.g.~sparse/low-rank approximation and generalized linear models \cite{cai2010singular,beck2009fast,jaganathan2012recovery,oymak2018sharp,agarwal2010fast}) and is left as a future work.

To derive the results from Section \ref{main result sec}, we need to determine the conditions under which Theorem \ref{det conv} is applicable to the data obtained from RNN state equation with high probability. Below we provide desirable characteristics of the state vector; which enables our statistical results.

\begin{assumption}[Well-behaved state vector] \label{ass well}Let $L>1$ be an integer. There exists positive scalars $\gamma_+,\gamma_-,\theta$ and an absolute constant $C>0$ such that $\theta\leq 3\sqrt{n}$ and the following holds
\begin{itemize}
\item {\bf{Lower bound:}} $\bSio{\h_{L-1}}\succeq \gamma_-\Iden_n$,
\item {\bf{Upper bound:}} for all $t$, the state vector satisfies
\begin{align}
\bSio{\h_t}\preceq \gamma_+\Iden_n\quad\text{,}\quad\tsub{\h_t-\E[\h_t]}\leq C\sqrt{\gamma_+}\quad\text{and}\quad\tn{\E[\h_t]}\leq \theta\sqrt{\gamma_+ }.\label{assupp bound}
\end{align}
Here $\tsub{\cdot}$ returns the subgaussian norm of a vector (see Definition \ref{ornorm}).
\end{itemize}
\end{assumption}
Assumption \ref{ass well} ensures that covariance is well-conditioned, state vector is well-concentrated, and it has a reasonably small expectation. Our next theorem establishes statistical guarantees for learning the RNN state equation based on this assumption.
\begin{theorem}[General result]\label{MAIN} Let $\{\ub_t,\h_{t+1}\}_{t=1}^N$ be a length $N$ trajectory of the state equation \eqref{main rel}. Suppose $\|\A\|<1$, $\phi$ is $\beta$-increasing, $\h_0=0$, and $\ub_t\distas\Nn(0,\Iden_p)$. Given scalars $\gamma_+\geq\gamma_->0$, set the condition number as $\rho=\gamma_+/\gamma_-$. For absolute constants $C,c,c_0>0$, choose trajectory length $N$ to satisfy
\[
N\geq CL \rho^2(n+p)\quad\text{where}\quad L=\lceil 1-\frac{\log{{(cn\rho)}{}}}{\log \|\A\|}\rceil.
\]
Suppose Assumption \ref{ass well} holds with $L,\gamma_+,\gamma_-,\theta$. Pick scaling to be $\mu=1/\sqrt{\gamma_+}$ and learning rate to be $\eta=c_0\frac{\beta^2}{\rho (\theta+\sqrt{2})^2(n+p)}$. With probability $1-4N\exp(-100n)-8L\exp(-\order{\frac{N}{L\rho^2}})$, starting from $\bTeta_0$, for all $\tau\geq 0$, the SGD iterations on loss \eqref{sgd loss} as described in Algorithm \ref{algo 1} satisfies
\begin{align}
\E[\tf{\bTeta_{\tau}-\Cb}^2]\leq (1-c_0\frac{\beta^4}{2\rho^2 (\theta+\sqrt{2})^2(n+p)})^{\tau}\tf{\bTeta_{0}-\Cb}^2,\label{conv bound}
\end{align}
where the expectation is over the randomness of SGD updates.
\end{theorem}
The advantage of this theorem is that, it isolates the optimization problem from the statistical properties of state vector. If one can prove tighter bounds on achievable $(\gamma_+,\gamma_-,\theta)$, it will immediately imply improved performance for SGD. In particular, Theorems \ref{main thm} and \ref{thm odd} are simple corollaries of Theorem \ref{MAIN} with proper choices.
\begin{itemize}
\item Theorem \ref{main thm} follows by setting $\gamma_+=B_{\infty}^2$, $\gamma_-=\beta^2\lmn{\B}^2$, and $\theta=\sqrt{n}$.
\item Theorem \ref{thm odd} follows by setting $\gamma_+=B_{\infty}^2$, $\gamma_-=\beta^2\lmn{\B}^2$, and $\theta=0$.
\end{itemize}

\section{Learning Unstable Systems}\label{sec unstab}
So far, we considered learning from a single RNN trajectory for stable systems ($\|\A\|<1$). For such systems, as the time goes on, the impact of the earlier states disappear. In our analysis, this allows us to split a single trajectory into multiple nearly-independent trajectories. This approach will not work for unstable systems ($\A$ is arbitrary) where the impact of older states may be amplified over time. To address this, we consider a model where the data is sampled from multiple independent trajectories.

Suppose $N$ independent trajectories of the state-equation \eqref{main rel} are available. Pick some integer $T_0\geq 1$. Denoting the $i$th trajectory by the triple $(\h^{(i)}_{t+1},\h^{(i)}_{t},\ub^{(i)}_{t})_{t\geq 0}$, we collect a single sample from each trajectory at time $T_0$ to obtain the triple $(\h^{(i)}_{T_0+1},\h^{(i)}_{T_0},\ub^{(i)}_{T_0})$. To utilize the existing optimization framework \eqref{sgd loss}; for $1\leq i\leq N$, we set,
\begin{align}
(\y_i,\h_{i},\ub_{i})=(\h^{(i)}_{T_0+1},\h^{(i)}_{T_0},\ub^{(i)}_{T_0}).\label{uns sample}
\end{align}
With this setup, we can again use the SGD Algorithm \ref{algo 1} to learn the weights $\A$ and $\B$. The crucial difference compared to Section \ref{main result sec} is that, the samples $(\y_{i},\h_{i},\ub_{i})_{i=1}^N$ are now independent of each other; hence, the analysis is simplified. As previously, having an upper bound on the condition number of the state-vector covariance is critical. This upper bound can be shown to be $\rho$ defined as
\begin{align}
\rho=\begin{cases}\bar{\rho}\quad\text{if}~n>1\\\bar{\rho}\frac{1-\beta^2|\A|^2}{1-(\beta|\A|)^{2T_0}}\quad\text{if}~n=1\end{cases}~~~\text{where}~~~\bar{\rho}=\frac{B_{T_0}^2}{\beta^2\lmn{\B}^2}.\label{uns rho}
\end{align}
The $\bar{\rho}$ term is similar to the earlier definition \eqref{rhodef}; however it involves $B_{T_0}$ rather than $B_\infty$. This modification is indeed necessary since $B_{\infty}=\infty$ when $\|\A\|>1$. On the other hand, note that, $B_{T_0}^2$ grows proportional to $\|\A\|^{2T_0}$; which results in exponentially bad condition number in $T_0$. Our $\rho$ definition remedies this issue for single-output systems; where $n=1$ and $\A$ is a scalar. In particular, when $\beta=1$ (e.g.~$\phi$ is linear) $\rho$ becomes equal to the correct value $1$\footnote{Clearly, any nonzero $1\times 1$ covariance matrix has condition number $1$. However, due to subtleties in the proof strategy, we don't use $\rho=1$ for $\beta<1$. Obtaining tighter bounds on the subgaussian norm of the state-vector would help resolve this issue.}. The next theorem provides our result on unstable systems in terms of this condition number and other model parameters.
 \begin{theorem} [Unstable systems]\label{thm unstab} Suppose we are given $N$ independent trajectories $(\h^{(i)}_{t},\ub^{(i)}_{t})_{t\geq 0}$ for $1\leq i\leq N$. Each trajectory is sampled at time $T_0$ to obtain $N$ samples $(\y_{i},\h_{i},\ub_{i})_{i=1}^N$ where the $i$th sample is given by \eqref{uns sample}.
 Suppose the sample size satisfies \[N\geq C \rho^2 (n+p)\] where $\rho$ is given by \eqref{uns rho}. Assume the initial states are $0$, $\phi$ is $\beta$-increasing, $p\geq n$, and $\ub_t\distas\Nn(0,\Iden_p)$. Set scaling $\mu=1/\sqrt{B_{T_0}}$, learning rate $\eta=c_0\frac{\beta^2}{\rho n(n+p)}$, and run SGD over the equations described in \eqref{reparam} and \eqref{sgd loss}. Starting from $\bTeta_0$, with probability $1-2N\exp(-100(n+p))-4\exp(-\order{\frac{N}{\rho^2}})$, all SGD iterations satisfy
\[
\E[\tf{\bTeta_{\tau}-\Cb}^2]\leq (1-c_0\frac{\beta^4}{2\rho^2n(n+p)})^{\tau}\tf{\bTeta_{0}-\Cb}^2,
\]
where the expectation is over the randomness of the SGD updates.
\end{theorem}


\section{Numerical Experiments}\label{numeric sec}
We did synthetic experiments on ReLU and Leaky ReLU activations. Let us first describe the experimental setup. We pick state dimension $n=50$ and input dimension $p=100$. We choose the ground truth matrix $\A$ to be a scaled random unitary matrix; which ensures that all singular values of $\A$ are equal. $\B$ is generated with i.i.d.~$\Nn(0,1)$ entries. Instead of using the theoretical scaling choice, we determine the scaling $\mu$ from empirical covariance matrices outlined in Algorithm \ref{algo 3}. Similar to our proof strategy, this algorithm equalizes the spectral norms of the input and state covariances to speed up convergence. We also empirically determined the learning rate and used $\eta=1/100$ in all experiments.
\begin{algorithm} [!t]\caption{Empirical hyperparameter selection.}\label{algo 3}
\begin{algorithmic}[1]
\item {\bf{Inputs:}} $(\h_{t},\ub_{t})_{t=1}^N$ sampled from a trajectory.
\item {\bf{Outputs:}} Scaling $\mu$.
\item Form the empirical covariance matrix $\Sigma_h$ from $\{\h_{t}\}_{t=1}^N$.
\item Form the empirical covariance matrix $\Sigma_u$ from $\{\ub_{t}\}_{t=1}^N$.\\
\Return $\sqrt{\|\Sigma_u\|/\|\Sigma_h\|}$.
\end{algorithmic}
\end{algorithm}
\begin{figure}[t!]
\begin{centering}
\begin{subfigure}[t]{3in}
\includegraphics[height=0.7\linewidth,width=1\linewidth]{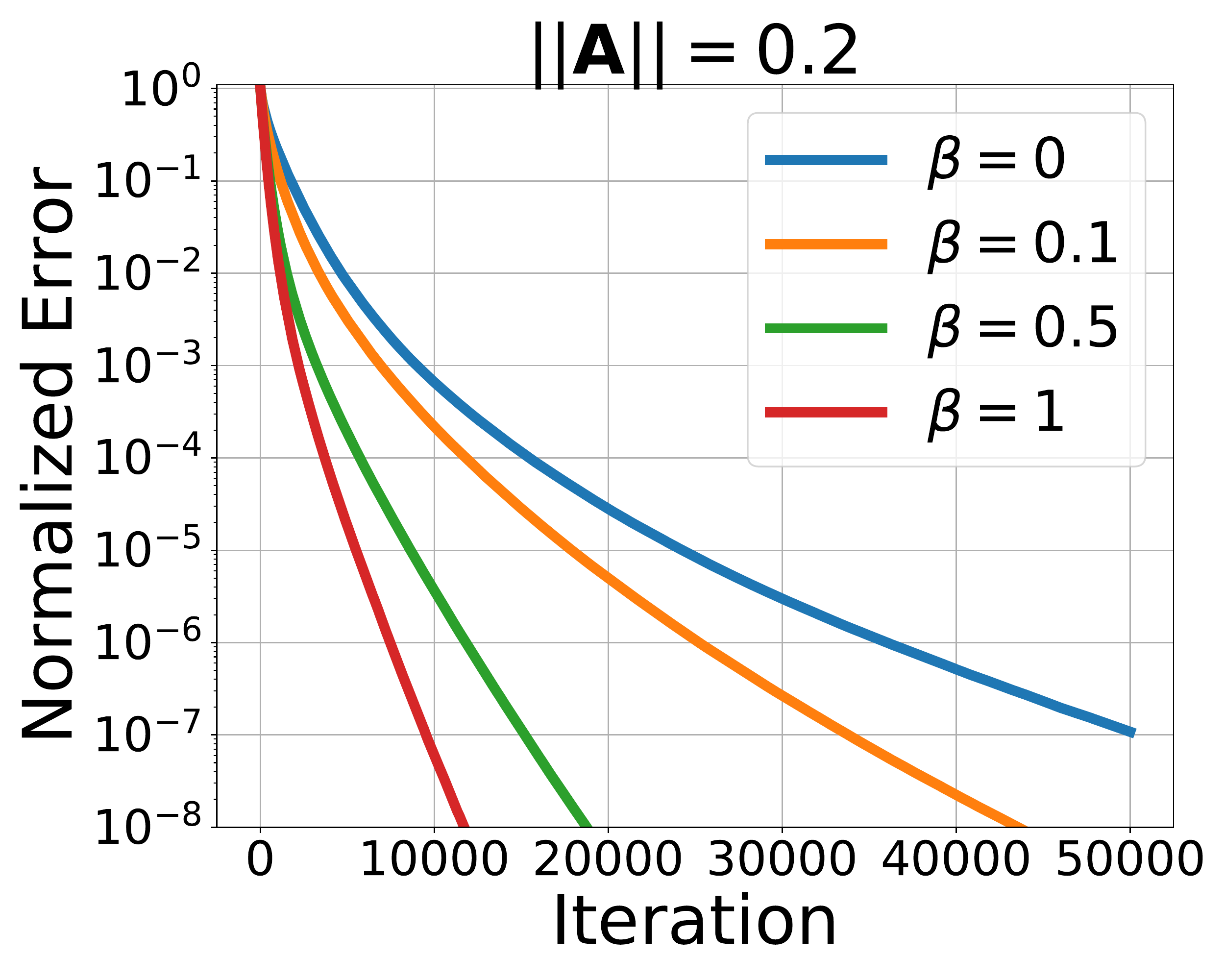}\vspace{-5pt}\subcaption{}\label{fig1a}
\end{subfigure}
\end{centering}~
\begin{centering}
\begin{subfigure}[t]{3in}
\includegraphics[height=0.7\linewidth,width=1\linewidth]{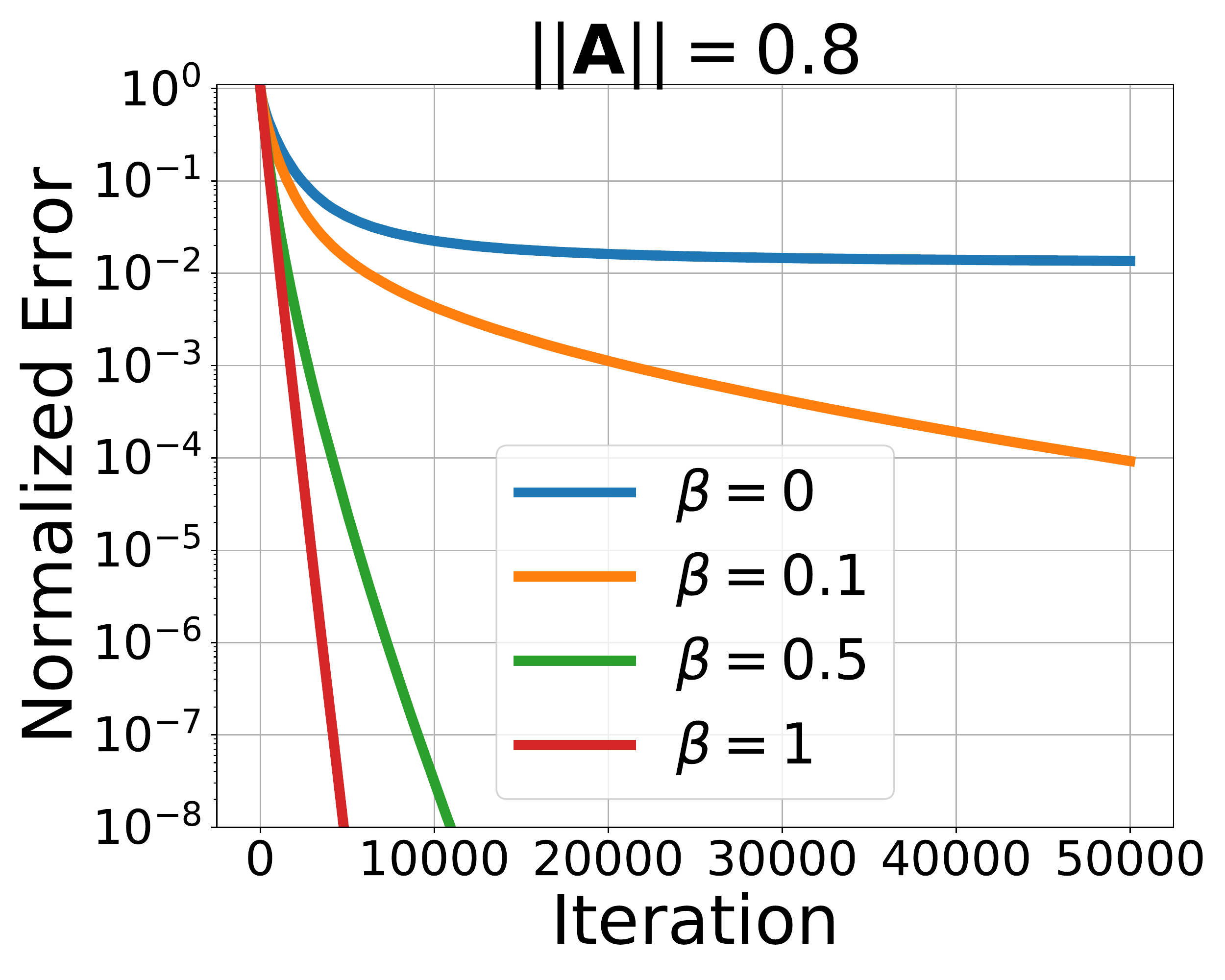}\vspace{-5pt}\subcaption{}\label{fig1b}
\end{subfigure}
\end{centering}\vspace{-10pt}\caption{SGD convergence behavior for Leaky ReLUs with varying minimum slope $\beta$. Figures a) and b) repeat the same experiments. The difference is the spectral norm of the ground truth state matrix $\A$.}\label{fig1}
\end{figure}

\noindent{\bf{Evaluation:}} We consider two performance measures in the experiments. Let $\hat{\Cb}$ be an estimate of the ground truth parameter $\Cb=[\mu^{-1}\A~\B]$. The first measure is the normalized error defined as $\frac{\tf{\hat{\Cb}-\Cb}^2}{\tf{\Cb}^2}$. The second measure is the normalized loss defined as
\[
\frac{\sum_{i=1}^N \tn{\y_t-\phi(\hat{\Cb}\x_t)}^2}{\sum_{i=1}^N \tn{\y_t}^2}.
\]
In all experiments, we run Algorithm \ref{algo 1} for $50000$ SGD iterations and plot these measures as a function of $\tau$; by using the estimate available at the end of the $\tau$th SGD iteration for $0\leq \tau\leq 50000$. Each curve is obtained by averaging the outcomes of 20 independent realizations.

Our first experiments use $N=500$; which is mildly larger than the total dimension $n+p=150$. In Figure \ref{fig1}, we plot Leaky ReLUs with varying slopes as described in \eqref{leaky relu}. Here $\beta=0$ corresponds to ReLU and $\beta=1$ is the linear model with identity activation. In consistence with our theory, SGD achieves linear convergence and as $\beta$ increases, the rate of convergence drastically improves. The improvement is more visible for less stable systems driven by $\A$ with a larger spectral norm. In particular, while ReLU converges for small $\|\A\|$, SGD gets stuck before reaching the ground truth when $\|\A\|=0.8$.

To understand, how well SGD fits the training data, in Figure \ref{fig2a}, we plotted the normalized loss for ReLU activation. For more unstable system ($\|\A\|=0.9$), training loss stagnates in a similar fashion to the parameter error. We also verified that the norm of the overall gradient $\tf{\grad{\Theta_{\tau}}}$ continues to decay (where $\Theta_{\tau}$ is the $\tau$th SGD iterate); which implies that SGD converges before reaching a global minima. As $\A$ becomes more stable, rate of convergence improves and linear rate is visible. Finally, to better understand the population landscape of the quadratic loss with ReLU activations, Figure \ref{fig2b} repeats the same ReLU experiments while increasing the sample size five times to $N=2500$. For this more overdetermined problem, SGD converges even for $\|\A\|=0.9$; indicating that 
\begin{itemize}
\item population landscape of loss with ReLU activation is well-behaved,
\item however ReLU problem requires more data compared to the Leaky ReLU for finding global minima.
\end{itemize}
Overall, as predicted by our theory, experiments verify that SGD indeed quickly finds the optimal weight matrices of the state equation \eqref{main rel} and as the activation slope $\beta$ increases, the convergence rate improves.

\begin{figure}[t!]
\begin{centering}
\begin{subfigure}[t]{3in}
\includegraphics[height=0.7\linewidth,width=1\linewidth]{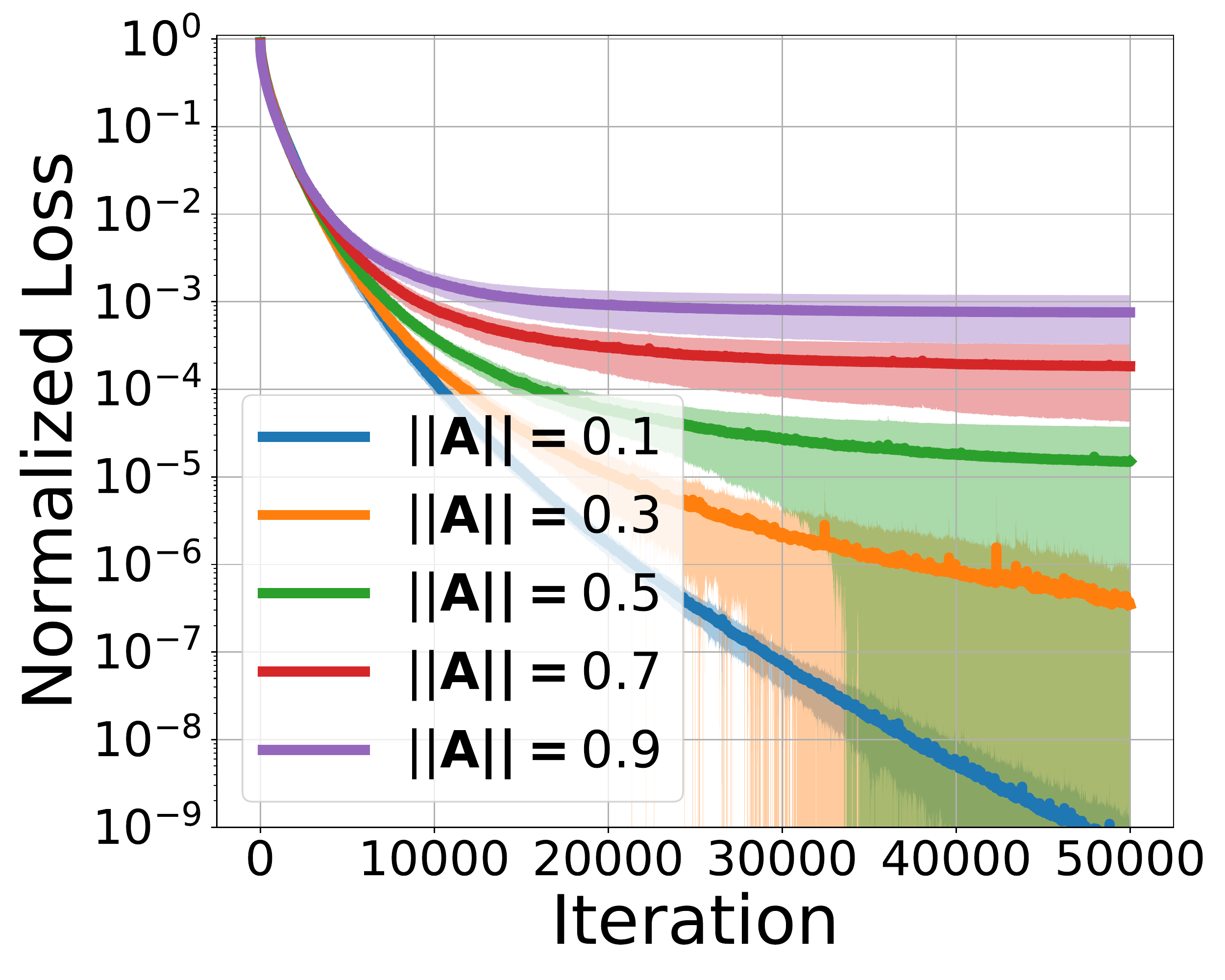}\vspace{-5pt}\subcaption{}\label{fig2a}
\end{subfigure}
\end{centering}~
\begin{centering}
\begin{subfigure}[t]{3in}
\includegraphics[height=0.7\linewidth,width=1\linewidth]{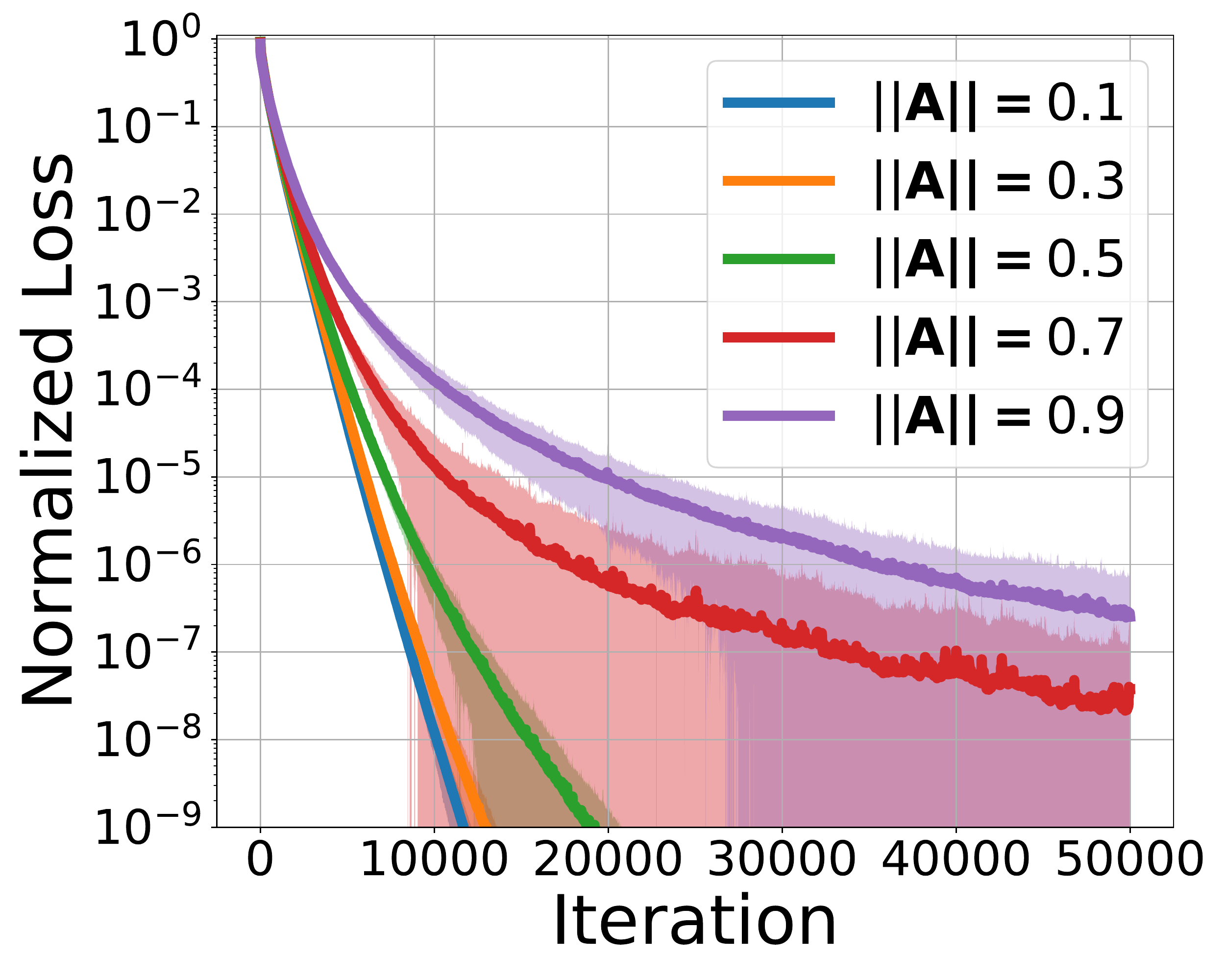}\vspace{-5pt}\subcaption{}\label{fig2b}
\end{subfigure}
\end{centering}\vspace{-10pt}\caption{SGD convergence behavior for ReLU with varying spectral norm of the state matrix $\A$. Figures a) and b) repeats the same experiments. The difference is that a) uses $N=500$ trajectory length whereas b) uses $N=2500$ (i.e.~$\times 5$ more data). Shaded regions highlight the one standard deviation around the mean.}\label{fig2}
\end{figure}

\section{Conclusions}
This work showed that SGD can learn the nonlinear dynamical system \eqref{main rel}; which is characterized by weight matrices and an activation function. This problem is of interest for recurrent neural networks as well as nonlinear system identification. We showed that efficient learning is possible with optimal sample complexity and good computational performance. Our results apply to strictly increasing activations such as Leaky ReLU. We empirically showed that Leaky ReLU converges faster than ReLU and requires less samples; in consistence with our theory. We list a few unanswered problems that would provide further insights into recurrent neural networks.
\begin{itemize}
\item {\bf{Covariance of the state-vector:}} Our results depend on the covariance of the state-vector and requires it to be positive definite. One might be able to improve the current bounds on the condition number and relax the assumptions on the activation function. Deriving similar performance bounds for ReLU is particularly interesting.
\item {\bf{Hidden state:}} For RNNs, the state vector is hidden and is observed through an additional equation \eqref{out observe}; which further complicates the optimization landscape. Even for linear dynamical systems, learning the $(\A,\B,\Cb,\Db)$ system (\eqref{main rel}, \eqref{out observe}) is a non-trivial task \cite{ho1966effective,hardt2016gradient}. What can be said when we add the nonlinear activations?
\item {\bf{Classification task:}} In this work, we used normally distributed input and least-squares regression for our theoretical guarantees. More realistic input distributions might provide better insight into contemporary problems, such as natural language processing; where the goal is closer to classification (e.g.~finding the best translation from another language). 
\end{itemize}
\section*{Acknowledgements}
We would like to thank Necmiye Ozay and Mahdi Soltanolkotabi for helpful discussions.
\small{
\bibliographystyle{plain}
\bibliography{Bibfiles}
}
\appendix
\section{Deterministic Convergence Result for SGD}

\begin{proof} [Proof of Theorem \ref{det conv}]Given two distinct scalars $a,b$; define $\phi'(a,b)=\frac{\phi(a)-\phi(b)}{a-b}$. $\phi'(a,b)\geq \beta$ since $\phi$ is $\beta$-increasing. Define $\w_\tau$ to be the residual $\w_\tau=\bteta_\tau-\bteta$. Observing 
\[
\phi(\x_{r_\tau}^T\bteta_\tau)-\y_{r_\tau}=\phi'(\x_{r_\tau}^T\bteta_\tau,\x_{r_\tau}^T\bteta)\x_{r_\tau}^T\w_\tau,
\]
the SGD recursion obeys
\begin{align}
\tn{\w_{\tau+1}}^2&=\tn{\w_\tau-\eta(\phi(\x_{r_\tau}^T\bteta_\tau)-\y_{r_\tau})\phi'(\x_{r_\tau}^T\bteta_\tau)\x_{r_\tau}}^2.\\
&=\tn{\w_\tau-\eta\x_{r_\tau}\phi'(\x_{r_\tau}^T\bteta_\tau)\phi'(\x_{r_\tau}^T\bteta_\tau,\x_{r_\tau}^T\bteta)\x_{r_\tau}^T\w_\tau}^2\\
&=\tn{(\Iden-\eta\Gb_{r_\tau})\w_\tau}^2
\end{align}
where $\Gb_{r_\tau}=\x_{r_\tau}\phi'(\x_{r_\tau}^T\bteta_\tau)\phi'(\x_{r_\tau}^T\bteta_\tau,\x_{r_\tau}^T\bteta)\x_{r_\tau}^T$. Since $\phi$ is $1$-Lipschitz and $\beta$-increasing, $\Gb_{r_\tau}$ is a positive-semidefinite matrix satisfying
\begin{align}
&\x_{r_\tau}\x_{r_\tau}^T\succeq\Gb_{r_\tau}\succeq\beta^2\x_{r_\tau}\x_{r_\tau}^T,\\
&\Gb_{r_\tau}^T\Gb_{r_\tau}\preceq \x_{r_\tau}\x_{r_\tau}^T\x_{r_\tau}\x_{r_\tau}^T\preceq B\x_{r_\tau}\x_{r_\tau}^T.
\end{align}
Consequently, we find the following bounds in expectation
\begin{align}
&\gamma_+\Iden_n\succeq \E[\Gb_{r_\tau}]\succeq \beta^2\gamma_-\Iden_n,\label{s convex}\\
&\E[\Gb_{r_\tau}^T\Gb_{r_\tau}]\preceq B\gamma_+\Iden_n.
\end{align}
Observe that \eqref{s convex} essentially lower bounds the {\em{strong convexity}} parameter of the problem with $\beta^2\gamma_-$; which is the strong convexity of the identical problem with the linear activation (i.e.~$\beta=1$). However, we only consider strong convexity around the ground truth parameter $\bteta$ i.e.~we restricted our attention to $(\bteta,\bteta_\tau)$ pairs. With this, $\w_{\tau+1}$ can be controlled as,
\begin{align}
\E[\tn{\w_{\tau+1}}^2]&=\E[\tn{(\Iden-\eta\Gb_{r_\tau})\w_\tau}^2]\\
&=\tn{\w_{\tau}}^2-2\eta \E[\w_{\tau}^T\Gb_{r_\tau}\w_\tau]+\eta^2 \E[\w_{\tau}^T\Gb_{r_\tau}^T\Gb_{r_\tau}\w_\tau]\\
&\leq \tn{\w_\tau}^2(1-2\eta\beta^2 \gamma_-+\eta^2B\gamma_+).
\end{align}
Setting $\eta=\frac{\beta^2\gamma_-}{\gamma_+B}$, we find the advertised bound
\[
\E[\tn{\w_{\tau+1}}^2]\leq \E[\tn{\w_\tau}^2](1-\frac{\beta^4\gamma_-^2}{\gamma_+B}).
\]
Applying induction over the iterations $\tau$, we find the advertised bound \eqref{adrec}
\[
\E[\tn{\w_\tau}^2]\leq \tn{\w_0}^2(1-\frac{\beta^4\gamma_-^2}{\gamma_+B})^{\tau}.
\]
\end{proof}

\begin{lemma} [Merging $L$ splits]\label{lem merge}Assume matrices $\X^{(i)}\in\R^{N_i\times q}$ are given for $1\leq i\leq L$. Suppose for all $1\leq i\leq L$, rows of $\X^{(i)}$ has $\ell_2$ norm at most $\sqrt{B}$ and each $\X^{(i)}$ satisfies
\[
\gamma_+\Iden_n\succeq\frac{{\X^{(i)}}^T\X^{(i)}}{N_i}\succeq\gamma_-\Iden_n.
\]
Set $N=\sum_{i=1}^L N_i$ and form the concatenated matrix $\X=\begin{bmatrix}\X^{(1)}\\\X^{(2)}\\\vdots\\\X^{(L)}\end{bmatrix}$. 
Denote $i$th row of $\X$ by $\x_i$. Then, for each $i$, $\tn{\x_i}^2\leq B$ and
\[
\gamma_+\Iden_n\succeq \frac{\X^T\X}{N}=\frac{1}{N}\sum_{i=1}^N \x_i\x_i^T\succeq \gamma_-\Iden_n.
\]
\end{lemma}
\begin{proof} The bound on the rows $\tn{\x_i}$ directly follows by assumption. For the remaining result, first observe that $\X^T\X=\sum_{i=1}^L{\X^{(i)}}^T\X^{(i)}$. Next, we have
\[
N\gamma_+\Iden_n=\sum_{i=1}^LN_i\gamma_+\Iden_n\succeq \sum_{i=1}^L{\X^{(i)}}^T\X^{(i)}\succeq \sum_{i=1}^LN_i\gamma_-\Iden_n=N\gamma_-\Iden_n.
\]
Combining these two yields the desired upper/lower bounds on ${\X^T\X}/{N}$.
\end{proof}

\section{Properties of the nonlinear state equations}\label{sec fundamental}
This section characterizes the properties of the state vector $\hh_t$ when input sequence is normally distributed. These bounds will be crucial for obtaining upper and lower bounds for the singular values of the data matrix $\X=[\x_1~\dots~\x_N]^T$ described in \eqref{reparam}. For probabilistic arguments, we will use the properties of subgaussian random variables. Orlicz norm provides a general framework that subsumes subgaussianity.
\begin{definition}[Orlicz norms] \label{ornorm}For a scalar random variable Orlicz-$a$ norm is defined as
\[
\|X\|_{\psi_{a}}=\sup_{k\geq 1}k^{-1/a}(\E[|X|^k])^{1/k}
\]
Orlicz-$a$ norm of a vector $\x\in\R^p$ is defined as $\|\x\|_{\psi_{a}}=\sup_{\vb\in \Bc^{p}} \|\vb^T\x\|_{\psi_{a}}$ where $\Bc^p$ is the unit $\el_2$ ball.
The subexponential norm is the Orlicz-$1$ norm $\te{\cdot}$ and the subgaussian norm is the Orlicz-$2$ norm $\tsub{\cdot}$.
\end{definition}

\begin{lemma}[Lipschitz properties of the state vector] \label{lip state}Consider the state equation \eqref{main rel}. Suppose activation $\phi$ is $1$-Lipschitz. Observe that $\h_{t+1}$ is a deterministic function of the input sequence $\{\ub_\tau\}_{\tau=0}^t$. Fixing all vectors $\{\ub_i\}_{i\neq \tau}$ (i.e.~all except $\ub_\tau$), $\h_{t+1}$ is $\|\A\|^{t-\tau}\|\B\|$ Lipschitz function of $\ub_\tau$ for $0\leq \tau\leq t$.
\end{lemma}
\begin{proof} Fixing $\{\ub_i\}_{i\neq \tau}$, denote $\hh_{t+1}$ as a function of $\ub_\tau$ by $\hh_{t+1}(\ub_\tau)$. Given a pair of vectors $\ub_\tau,\ub'_\tau$ using $1$-Lipschitzness of $\phi$, for any $t> \tau$, we have
\begin{align*}
\tn{\hh_{t+1}(\ub_\tau)-\hh_{t+1}(\ub'_\tau)}&\leq \tn{\phi(\A\h_t(\ub_\tau)+\B\ub_t)-\phi(\A\h_t(\ub'_\tau)+\B\ub_t)}\\
&\leq \tn{\A(\h_t(\ub_\tau)-\h_t(\ub'_\tau))}\\
&\leq \|\A\|\tn{\h_t(\ub_\tau)-\h_t(\ub'_\tau)}.
\end{align*}
Proceeding with this recursion until $t=\tau$, we find
\begin{align*}
\tn{\hh_{t+1}(\ub_\tau)-\hh_{t+1}(\ub'_\tau)}&\leq \|\A\|^{t-\tau}\tn{\h_{\tau+1}(\ub_\tau)-\h_{\tau+1}(\ub'_\tau)}\\
&\leq \|\A\|^{t-\tau}\tn{\phi(\A\h_\tau+\B\ub_\tau)-\phi(\A\h_\tau+\B\ub'_\tau)}\\
&\leq \|\A\|^{t-\tau}\|\B\|\tn{\ub_\tau-\ub'_\tau}.
\end{align*}
This bound implies $\hh_{t+1}(\ub_\tau)$ is $\|\A\|^{t-\tau}\|\B\|$ Lipschitz function of $\ub_\tau$.
\end{proof}

\newcommand\mysim{\stackrel{\mathclap{\normalfont\mbox{\footnotesize{i.i.d.}}}}{\sim}}

\begin{lemma} [Upper bound]\label{upp bound}Consider the state equation governed by equation \eqref{main rel}. Suppose $\ub_t\distas\Nn(0,\Iden_p)$, $\phi$ is $1$-Lipschitz, $\phi(0)=0$ and $\h_0=0$. Recall the definition \eqref{bteq} of $\lip_t$. We have the following properties
\begin{itemize}
\item $\hh_{t}$ is a $\lip_t$-Lipschitz function of the vector $\qb_t=[\ub_0^T~\dots~\ub_{t-1}^T]^T\in\R^{tp}$.
\item There exists an absolute constant $c>0$ such that $\tsub{\hh_{t}-\E[\hh_{t}]}\leq c\lip_t$ and $\bSio{\h_{t}}\preceq \lip_t^2\Iden_n$.
\item $\h_{t}$ satisfies
\[
\E[\tn{\hh_{t}}^2]\leq  \tr{\B\B^T}\frac{1-\|\A\|^{2t}}{1-\|\A\|^2}\leq \min\{n,p\}B_t^2.
\]
Also, there exists an absolute constant $c>0$ such that for any $m\geq n$, with probability $1-2\exp(-100m)$, $\tn{\hh_t}\leq c\sqrt{m}\lip_t$.
\end{itemize}
\end{lemma}
\begin{proof} 
\noindent{\bf{i) Bounding Lipschitz constant:}} Observe that $\h_{t}$ is a deterministic function of $\qb_t$ i.e.~$\h_{t}=f(\qb_t)$ for some function $f$. To bound Lipschitz constant of $f$, for all (deterministic) vector pairs $\qb_t$ and $\hat{\qb}_t$, we find a scalar $L_f$ satisfying,
\begin{align}
\tn{f(\qb_t)-f(\hat{\qb}_t)}\leq L_f\tn{\qb_t-\hat{\qb}_t}.\label{liplip}
\end{align}
Define the vectors, $\{\ab_i\}_{i=0}^{t}$, as follows
\begin{align}
\ab_i=[\hat{\ub}_0^T~\dots~\hat{\ub}_{i-1}^T~{\ub}_{i}^T~\dots~{\ub}_{t-1}^T]^T.\nn
\end{align}
Observing that $\ab_0=\qb_t$, $\ab_{t}=\hat{\qb}_t$, we write the telescopic sum,
\[
\tn{f(\qb_t)-f(\hat{\qb}_t)}\leq \sum_{i=0}^{t-1} \tn{f(\ab_{i+1})-f(\ab_{i})}.
\]
Focusing on the individual terms $f(\ab_{i+1})-f(\ab_{i})$, observe that the only difference is the $\ub_i,\hat{\ub}_i$ terms. Viewing $\h_{t}$ as a function of $\ub_i$ and applying Lemma \ref{lip state}, 
\begin{align}
\tn{f(\ab_{i+1})-f(\ab_{i})}\leq \|\A\|^{t-1-i}\|\B\|\tn{\ub_i-\hat{\ub}_i}.\nn
\end{align}
To bound the sum, we apply the Cauchy-Schwarz inequality; which yields
\begin{align}
|f(\qb_t)-f(\hat{\qb}_t)|&\leq \sum_{i=0}^{t-1}\|\A\|^{t-1-i}\|\B\|\tn{\ub_i-\hat{\ub}_i}\nn\\
&\leq {{(\sum_{i=0}^{t-1}\|\A\|^{2(t-1-i)}\|\B\|^2)^{1/2}}}\underbrace{(\sum_{i=0}^{t-1}{\tn{\ub_i-\hat{\ub}_i}^2})^{1/2}}_{\tn{\qb_t-\hat{\qb}_t}}\nn\\
&\leq {{\sqrt{\frac{\|\B\|^2(1-\|\A\|^{2t})}{1-\|\A\|^{2}}}}}\tn{\qb_t-\hat{\qb}_t}\nn\\
&= B_t\tn{\qb_t-\hat{\qb}_t}.
\end{align}
The final line achieves the inequality \eqref{liplip} with $L_f=B_{t}$ hence $\h_{t}$ is $\lip_{t}$ Lipschitz function of $\qb_t$.
%

\noindent{\bf{ii) Bounding subgaussian norm:}} When $\ub_t\distas\Nn(0,\Iden_p)$, the vector $\qb_t$ is distributed as $\Nn(0,\Iden_{tp})$. Since $\h_{t}$ a $\lip_t$ Lipschitz function of $\qb_t$, for any fixed unit length vector $\vb$, $\alpha_{\vb}:=\vb^T\h_{t}=\vb^Tf(\qb_t)$ is still $\lip_t$-Lipschitz function of $\qb_t$. Hence, using Gaussian concentration of Lipschitz functions, $\alpha_{\vb}$ satisfies
\[
\Pro(|\alpha_{\vb}-\E[\alpha_{\vb}]|\geq t)\leq 2\exp(-\frac{t^2}{2\lip_t^2}).
\] 
This implies that for any $\vb$, $\alpha_{\vb}-\E[\alpha_{\vb}]$ is $\order{\lip_t}$ subgaussian \cite{vershynin2010introduction}. This is true for all unit $\vb$, hence using Definition \ref{ornorm}, the vector $\h_{t}$ satisfies $\tsub{\h_{t}-\E[\h_{t}]}\leq \order{\lip_t}$ as well. Secondly, $\lip_t$-Lipschitz function of a Gaussian vector obeys the variance inequality $\var{\alpha_{\vb}}\leq \lip_t^2$ (page $49$ of \cite{ledoux2001concentration}), which implies the covariance bound
\[
\bSio{\h_{t}}\preceq {\lip_t^2}\Iden_n.
\]
\noindent{\bf{iii) Bounding $\ell_2$-norm:}} To obtain this result, we first bound $\E[\tn{\hh_t}^2]$. Since $\phi$ is $1$-Lipschitz and $\phi(0)=0$, we have the deterministic relation
\[
\tn{\hh_{t+1}}\leq \tn{\A\hh_t+\B\ub_t}.
\]
Taking squares of both sides, expanding the right hand side, and using the independence of $\hh_t,\ub_t$ and the covariance information of $\ub_t$, we obtain
\begin{align}
\E[\tn{\hh_{t+1}}^2]&\leq \E[\tn{\A\hh_t+\B\ub_t}^2]=\E[\tn{\A\hh_t}^2]+\E[\tn{\B\ub_t}^2]\\
&\leq \|\A\|^2\E[\tn{\hh_t}^2]+\tr{\B\B^T}.
\end{align}
Now that the recursion is established, expanding $\hh_t$ on the right hand side until $\hh_0=0$, we obtain
\[
\E[\tn{\hh_{t+1}}^2]\leq \sum_{i=0}^t\|\A\|^{2i}\tr{\B\B^T}\leq\tr{\B\B^T} \frac{1-\|\A\|^{2(t+1)}}{1-\|\A\|^2}.
\]
Now using the fact that $\tr{\B\B^T}\leq\text{rank}(\B)\|\B\|^2\leq \min\{n,p\}\|\B\|^2$, we find
\[
\E[\tn{\hh_{t+1}}]^2\leq \E[\tn{\hh_{t+1}}^2]\leq \min\{n,p\}\lip_{t+1}^2.
\]
Finally, using the fact that $\hh_{t}$ is $\lip_t$-Lipschitz function and utilizing Gaussian concentration of $\qb_t\sim\Nn(0,\Iden_{tp})$, we find
\[
\Pro(\tn{\hh_{t+1}}-\E[\tn{\hh_{t+1}}]\geq t)\leq \exp(-\frac{t^2}{2\lip_t^2}).
\]
Setting $t=(c-1)\sqrt{m}\lip_t$ for sufficiently large $c>0$, we find $\Pro(\tn{\hh_{t}}\geq \sqrt{n}B_t+(c-1)\sqrt{m}\lip_t)\leq \exp(-100m)$.
%
\end{proof}

\begin{lemma}[Odd activations] \label{lem odd}Suppose $\phi$ is strictly increasing and obeys $\phi(x)=-\phi(-x)$ for all $x$ and $\h_0=0$. Consider the state equation \eqref{main rel} driven $\ub_t\distas\Nn(0,\Iden_p)$. We have that $\E[\h_t]=0$.
\end{lemma}
\begin{proof} We will inductively show that $\{\h_t\}_{t\geq 0}$ has a symmetric distribution around $0$. Suppose the vector $\h_t$ satisfies this assumption. Let $S\subset\R^n$ be a set. We will argue that $\Pro(\h_{t+1}\subset S)=\Pro(\h_{t+1}\subset -S)$. Since $\phi$ is strictly increasing, it is bijective on vectors, and we can define the unique inverse set $S'=\phi^{-1}(S)$. Also since $\phi$ is odd, $\phi(-S')=-S$. Since $\h_t,\ub_t$ are independent and symmetric, we reach the desired conclusion as follows
\begin{align}
\Pro(\h_{t+1}\subset S)&=\Pro(\A\h_t+\B\ub_t\subset S')=\Pro(\A(-\h_t)+\B(-\ub_t)\subset S')\\
&=\Pro(\A\h_t+\B\ub_t\subset -S')=\Pro(\phi(\A\h_t+\B\ub_t)\subset \phi(-S'))=\Pro(\h_{t+1}\subset -S).
\end{align}
\end{proof}
\begin{theorem} [State-vector lower bound] \label{lwbnd1}Consider the nonlinear state equation \eqref{main rel} with $\{\ub_t\}_{t\geq 0}\distas\Nn(0,\Iden_p)$. Suppose $\phi$ is a $\beta$-increasing function for some constant $\beta>0$. For any $t\geq 1$, the state vector obeys 
\[
\bSio{\h_{t}}\succeq\beta^2\smn(\B\B^T)\Iden_n.\label{h low u}
\]\end{theorem}
\begin{proof} The proof is an application of Lemma \ref{vec low}. The main idea is to write $\h_t$ as sum of two independent vectors, one of which has independent entries. Consider a multivariate Gaussian vector $\g\sim\Nn(0,\bSi)$. $\g$ is statistically identical to $\g_1+\g_2$ where $\g_1\sim\Nn(0,\lmn{\bSi}\Iden_d)$ and $\g_2 \sim\Nn(0,\bSi-\lmn{\bSi}\Iden_d)$ are independent multivariate Gaussians.

Since $\B\ub_t\sim\Nn(0,\B\B^T)$, setting $\bSi=\B\B^T$ and $\smn=\lmn{\bSi}$, we have that $\B\ub_t\sim\g_1+\g_2$ where $\g_1,\g_2$ are independent and $\g_1\sim\Nn(0,\smn\Iden_n)$ and $\g_2\sim\Nn(0,\bSi-\smn\Iden_n)$. Consequently, we may write
\[
\B\ub_t+\A\h_t\sim\g_1+\g_2+\A\h_t.
\]
For lower bound, the crucial component will be the $\g_1$ term; which has i.i.d.~entries. Applying Lemma \ref{vec low} by setting $\x=\g_1$ and $\y=\g_2+\A\h_t$, and using the fact that $\h_t,\g_1,\g_2$ are all independent of each other, we find the advertised bound, for all $t\geq 0$, via
\[
\bSio{\h_{t+1}}=\bSio{\phi(\g_1+\g_2+\A\h_t)}\succeq\beta^2\smn\Iden_n.
\]
\end{proof}
The next theorem applies to multiple-input-single-output (MISO) systems where $\A$ is a scalar and $\B$ is a row vector. The goal is refining the lower bound of Theorem \ref{lwbnd1}.
\begin{theorem}[MISO lower bound] \label{miso lem}Consider the setup of Theorem \ref{lwbnd1} with single output i.e.~$n=1$. For any $t\geq 1$, the state vector obeys 
\[
\var{\h_{t}}\geq \beta^2\tn{\B}^2\frac{1-(\beta |\A|)^{2t}}{1-\beta^2|\A|^2}.\label{h low u2}
\]
\end{theorem}
\begin{proof} For any random variable $X$, applying Lemma \ref{vec low}, we have $\var{\phi(X)}\geq \beta^2\var{X}$. Recursively, this yields
\[
\var{\h_{t+1}}=\var{\phi(\A\h_{t}+\B\ub_t)}\geq \beta^2\var{\A\h_{t}+\B\ub_t}=\beta^2(|\A|^2\var{\h_t}+\tn{\B}^2).
\]
Expanding these inequalities till $\h_0$, we obtain the desired bound
\[
\var{\h_{t}}\geq \sum_{i=1}^t(\beta^{i} |\A|^{i-1}\tn{\B})^2.
\]
\end{proof}
\begin{lemma}[Vector lower bound] \label{vec low}Suppose $\phi$ is a $\beta$-increasing function. Let $\x=[\x_1~\dots~\x_n]^T$ be a vector with i.i.d.~entries distributed as $\x_i\sim X$. Let $\y$ be a random vector independent of $\x$. Then,
\[
\bSio{\phi(\x+\y)}\succeq \beta^2\var{X}\Iden_n.
\]
\end{lemma}
\begin{proof} We first apply law of total covariance (e.g.~Lemma \ref{totcov}) to simplify the problem using the following lower bound based on the independence of $\x$ and $\y$,
\begin{align}
\bSio{\phi(\x+\y)}&\succeq \E_{\y}[\bSi[\phi(\x+\y)\bgl\y]]\\
&=\E_{\y}[\bSi_{\x}[\phi(\x+\y)]].
\end{align}
Now, focusing on the covariance $\bSi_{\x}[\phi(\x+\y)]$, fixing a realization of $\y$, and using the fact that $\x$ has i.i.d.~entries; $\phi(\x+\y)$ has independent entries as $\phi$ applies entry-wise. This implies that $\bSi_{\x}[\phi(\x+\y)]$ is a diagonal matrix. Consequently, its lowest eigenvalue is the minimum variance over all entries,
\[
\bSi_{\x}[\phi(\x+\y)]\succeq \min_{1\leq i\leq n} \var{\phi(\x_i+\y_i)}\Iden_n.
\]
Fortunately, Lemma \ref{scalar lem} provides the lower bound $\var{\phi(\x_i+\y_i)}\geq \beta^2\var{X}$. Since this lower bound holds for any fixed realization of $\y$, it still holds after taking expectation over $\y$; which concludes the proof.
\end{proof}
The next two lemmas are helper results for Lemma \ref{vec low} and are provided for the sake of completeness.
\begin{lemma}[Law of total covariance] \label{totcov}Let $\x,\y$ be two random vectors and assume $\y$ has finite covariance. Then
\[
\bSio{\y}=\E[\bSio{\y\bgl\x}]+\bSio{\E[\y\bgl\x]}.
\]
\end{lemma}
\begin{proof} First, write $\bSio{\y}=\E[\y\y^T]-\E[\y]\E[\y^T]$. Then, applying the law of total expectation to each term,
\[
\bSio{\y}=\E[\E[\y\y^T\bgl\x]]-\E[\E[\y\bgl\x]]\E[\E[\y^T\bgl\x]].
\]
Next, we can write the conditional expectation as $\E[\E[\y\y^T\bgl\x]]=\E[\bSio{\y\bgl\x}]+\E[\E[\y\bgl\x]\E[\y\bgl\x]]^T$. To conclude, we obtain the covariance of $\E[\y\bgl\x]$ via the difference,
\[
\E[\E[\y\bgl\x]\E[\y\bgl\x]]^T-\E[\E[\y\bgl\x]]\E[\E[\y^T\bgl\x]]=\bSio{\E[\y\bgl\x]},
\]
which yields the desired bound.
\end{proof}
\begin{lemma}[Scalar lower bound]\label{scalar lem} Suppose $\phi$ is a $\beta$-increasing function with $\beta>0$ as defined in Definition \ref{betain}. Given a random variable $X$ and a scalar $y$, we have
\[
\var{\phi(X+y)}\geq \beta^2\var{X}.
\]
\end{lemma}
\begin{proof} Since $\phi$ is $\beta$-increasing, it is invertible and $\phi^{-1}$ is strictly increasing. Additionally, $\phi^{-1}$ is $1/\beta$ Lipschitz since,
\[
|\phi(a)-\phi(b)|\geq \beta |a-b|\implies |a-b|\geq \beta |\phi^{-1}(a)-\phi^{-1}(b)|.
\]
Using this observation and the fact that $\E[X]$ minimizes $\E(X-\alpha)^2$ over $\alpha$, $\var{\phi(X+y)}$ can be lower bounded as follows
\begin{align*}
\var{\phi(X+y)}&= \E(\phi(X+y)-\E[\phi(X+y)])^2\\
&\geq \beta^2\E((X+y)-\phi^{-1}(\E[\phi(X+y)]))^2\\
&\geq \beta^2\E(X+y-\E[X+y])^2\\
&=\beta^2\E(X-\E X)^2=\beta^2\var{X}.
\end{align*}
Note that, the final line is the desired conclusion.
\end{proof}

\section{Truncating Stable Systems}\label{sec trunc}
One of the challenges in analyzing dynamical systems is the fact that samples from the same trajectory have temporal dependence. This section shows that, for stable systems, the impact of the past states decay exponentially fast and the system can be approximated by using the recent inputs only. We first define the truncation of the state vector.
\begin{definition} [Truncated state vector] \label{trst1} Suppose $\phi(0)=0$, initial condition $h_0=0$, and consider the state equation \eqref{main rel}. Given a timestamp $t$, $\TK$-truncation of the state vector $\h_{t}$ is denoted by $\hbr_{t,\TK}$ and is equal to $\qb_t$ where
\begin{align}
\qb_{\tau+1}=\phi(\A\qb_\tau+\B\ub'_\tau)\quad,\quad q_0=0
\end{align}
 is the state vector generated by the inputs $\ub'_\tau$ satisfying
\[
\ub'_\tau=\begin{cases}0~\text{if}~\tau< t-\TK\\\ub_\tau~\text{else}\end{cases}.
\]
\end{definition}
In words, $\TK$ truncated state vector $\hbr_{t,\TK}$ is obtained by unrolling $\h_{t}$ until time $t-\TK$ and setting the contribution of the state vector $\h_{t-\TK}$ to $0$. This way, $\hbr_{t,\TK}$ depends only on the variables $\{\ub_\tau\}_{\tau=t-\TK}^{t-1}$.

The following lemma states that impact of truncation can be made fairly small for stable systems ($\|\A\|<1$).
\begin{lemma}[Truncation impact -- deterministic] \label{trunc det}Consider the state vector $\h_t$ and its $\TK$-truncation $\hbr_{t,\TK}$ from Definition \ref{trst1}. Suppose $\phi$ is $1$-Lipschitz. We have that
\[
\tn{\h_t-\hbr_{t,\TK}}\leq \begin{cases}0~\text{if}~t\leq L\\\|\A\|^{\TK}\tn{\h_{t-\TK}}~\text{else}\end{cases}.
\]
\end{lemma}
\begin{proof} 
When $t\leq \TK$, Definition \ref{trst1} implies $\ub'_\tau=\ub_\tau$ hence $\h_t=\qb_t=\hbr_{t,\TK}$. When $t>\TK$, we again use Definition \ref{trst1} and recall that $\ub'_\tau=0$ until time $\tau=t-\TK-1$. For all $t-\TK< \tau \leq t$, using $1$-Lipschitzness of $\phi$, we have that
\begin{align}
\tn{\h_\tau-\qb_\tau}&=\tn{\phi(\A\h_{\tau-1}+\B\ub_{\tau-1})-\phi(\A\qb_{\tau-1}+\B\ub_{\tau-1})}\nn\\
&\leq\tn{(\A\h_{\tau-1}+\B\ub_{\tau-1})-(\A\qb_{\tau-1}+\B\ub_{\tau-1})}\nn\\
&\leq \tn{\A(\h_{\tau-1}-\qb_{\tau-1})}\leq \|\A\|\tn{\h_{\tau-1}-\qb_{\tau-1}}.\nn
\end{align}
Applying this recursion between $t-\TK< \tau \leq t$ and using the fact that $\qb_{t-\TK}=0$ implies the advertised result
\begin{align}
\tn{\h_t-\qb_t}&\leq \|\A\|^{\TK}\tn{\h_{t-\TK}-\qb_{t-\TK}}\nn\\
&\leq \|\A\|^{\TK}\tn{\h_{t-\TK}}.\nn
\end{align}
\end{proof}
\subsection{Near independence of sub-trajectories}
We will now argue that, for stable systems, a single trajectory can be split into multiple nearly independent trajectories. First, we describe how the sub-trajectories are constructed.
\begin{definition}[Sub-trajectory]\label{ltau_sample}Let sampling rate $L\geq 1$ and offset $1\leq \tao\leq L$ be two integers. Let $\NK=\NK_{\tao}$ be the largest integer obeying $(\NK-1)L+\tao\leq N$. We sample the trajectory $\{\h_t,\ub_t\}_{t=0}^N$ at the points $\tao,\tao+L,\dots,\tao+(\NK-1)L+\tao$ and define the $\tao$th sub-trajectory as
\begin{align}
(\hi{i},\uu{i}):=(\hi{i,\tao},\uu{i,\tao})=(\h_{(i-1)L+\tao},\ub_{(i-1)L+\tao}).\nn
\end{align}
\end{definition}

\begin{definition} [Truncated sub-trajectory] \label{trst}Consider the state equation \eqref{main rel} and recall Definition \ref{trst1}. Given offset $\tao$ and sampling rate $L$, for $1\leq i\leq \NK$, the $i$th truncated sub-trajectory states are $\{\hb{i}\}_{i=1}^{\NK}$ where the $i$th state is defined as
\[
\hb{i}=\hbr_{\TK(i-1)+\tao,\TK-1}.
\]
\end{definition}
The truncated samples are independent of each other as shown in the next lemma.
\begin{lemma}\label{lem all indep} Consider the truncated states of Definition \ref{trst}. If \eqref{main rel} is generated by independent vectors~$\{\ub_t\}_{t\geq 0}$, for any offset $\tao$ and sampling rate $L$, the vectors $\{\hb{i}\}_{i=1}^\NK,\{\uu{i}\}_{i=1}^\NK$ are all independent of each other.
\end{lemma}
\begin{proof} By construction $\hb{i}$ only depends on the vectors $\{\ub_\tau\}_{\tau=\TK(i-2)+\tao+1}^{\TK(i-1)+\tao-1}$. Note that the dependence ranges $[\TK(i-2)+\tao+1,\TK(i-1)+\tao-1]$ are disjoint intervals for different $i$'s; hence $(\hb{i})_{i=1}^\NK$ are independent of each other. To show the independence of $\uu{i}$ and $\hb{i}$; observe that inputs $\uu{i}=\ub_{\TK(i-1)+\tao}$ have timestamp $\tao$ modulo $\TK$; which is not covered by the dependence range of $(\hb{i})_{i=1}^\NK$.
\end{proof}

If the input is randomly generated, Lemma \ref{trunc det} can be combined with a probabilistic bound on $\h_t$, to show that truncated states $\hb{i}$ are fairly close to the actual states $\hi{i}$.
\begin{lemma}[Truncation impact -- random] \label{trun lem rand}Given offset $\tao$ and sampling rate $L$, consider the state vectors of the sub-trajectory $\{\hi{i}\}_{i=1}^\NK$ and $L-1$-truncations $(\hb{i})_{i=1}^\NK$. Suppose $\{\ub_t\}_{t\geq 0}\distas\Nn(0,\Iden_p)$, $\|\A\|<1$, $\h_0=0$, $\phi$ is $1$-Lipschitz, and $\phi(0)=0$. Also suppose upper bound \eqref{assupp bound} of Assumption \ref{ass well} holds for some $\theta\leq\sqrt{n},\gamma_+>0$. There exists an absolute constant $c>0$ such that with probability at least $1-2\NK\exp(-100n)$, for all $1\leq i\leq \NK$, the following bound holds
\[
\tn{\hi{i}-\hb{i}}\leq c\sqrt{n}\|\A\|^{\TK-1}\sqrt{\gamma_+}.
\]
In particular, we can always pick $\gamma_+=B_\infty^2$ (via Lemma \ref{upp bound}).
\end{lemma}
\begin{proof} Using Assumption \ref{ass well}, we can apply Lemma \ref{lemma sublen} on vectors $\{\h_{(i-2)\TK+\tao+1}\}_{i=1}^{\NK}$. Using a union bound, with desired probability, all vectors obey
\[
\tn{\h_{(i-2)\TK+\tao+1}-\E[\h_{(i-2)\TK+\tao+1}]}\leq (c-1)\sqrt{n\gamma_+},
\]
for sufficiently large $c$. Since $\theta\leq\sqrt{n}$, triangle inequality implies $\tn{\h_{(i-2)\TK+\tao+1}}\leq c\sqrt{n\gamma_+}$. Now, applying Lemma \ref{trunc det}, for all $1\leq i\leq \NK$, we find
\begin{align*}
\tn{\hi{i}-\hb{i}}&=\tn{\h_{(i-1)\TK+\tao}-\hbr_{(i-1)\TK+\tao,\TK-1}}\\
&\leq \|\A\|^{\TK-1}\tn{\h_{(i-2)\TK+\tao+1}}\\
&\leq c\|\A\|^{\TK-1}\sqrt{n\gamma_+}.
\end{align*}
\end{proof}

\section{Properties of the data matrix}
This section utilizes the probabilistic estimates from Section \ref{sec fundamental} to provide bounds on the condition number of data matrices obtained from the RNN trajectory \eqref{main rel}. Following \eqref{reparam}, these matrices $\Hb,\Ub$ and $\X$ are defined as
\begin{align}
\Hb=[\h_1~\dots~\h_N]^T\dquad\Ub=\Hb=[\ub_1~\dots~\ub_N]^T\dquad \X=[\x_1~\dots~\x_N]^T.\label{matxxx}
\end{align}
The challenge is that, the state matrix $\Hb$ has dependent rows; which will be addressed by carefully splitting the trajectory $\{\ub_t,\h_t\}_{t=0}^N$ into multiple sub-trajectories which are internally weakly dependent as discussed in Section \ref{sec trunc}. We first define the matrices obtained from these sub-trajectories.
\begin{definition} \label{subsmat} Given sampling rate $L$ and offset $\tao$, consider the $L$-subsampled trajectory $\{\hh^{(i)},\ub^{(i)}\}_{i=1}^{\NK}$ as described in Definitions \ref{ltau_sample} and \ref{trst}. Define the matrices $\Hbb=\Hbb^{(\tao)}\in\R^{\NK\times n},~\Hc=\Hc^{(\tao)}\in\R^{\NK\times n},~\Uc=\Uc^{(\tao)}\in\R^{\NK\times p}$, and $\Xc=\Xc^{(\tao)}\in\R^{\NK\times (n+p)}$ as
\[
\Hbb=[\hb{1}~\dots~\hb{\NK}]^T,~\Hc=[\hh^{(1)}~\dots~\hh^{(\NK)}]^T,~\Uc=[\ub^{(1)}~\dots~\ub^{(\NK)}]^T,~\Xc=[\mu\Hc~\Uc].
\]
\end{definition}

\begin{lemma}[Handling perturbation]  \label{dmc thm}Consider the nonlinear state equation \eqref{main rel}. Given sampling rate $L>0$ and offset $\tao$, consider the matrices $\Hbb,\Hc,\Xc$ of Definition \ref{subsmat} and let $\Qb=[\gamma_+^{-1/2}\Hbb~\Uc]\in\R^{\NK\times (n+p)}$. Suppose Assumption \ref{ass well} holds, $\phi$ is $\beta$-increasing, and $\ub_t\distas\Nn(0,\Iden_p)$. There exists an absolute constant $C>0$ such that if $\NK\geq C\frac{\gamma_+^2}{\gamma_-^2}(n+p)$, with probability $1-8\exp(-c\frac{\gamma_-^2}{\gamma_+^2}\NK)$, for all matrices $\M$ obeying $\|\M-\Hbb\|\leq {\frac{\sqrt{\gamma_-\NK}}{10}}$, the perturbed $\Qb$ matrices given by,
\begin{align}
\Qbt=[\gamma_+^{-1/2}\M~\Uc],\label{q mat}
\end{align}
satisfy
\begin{align}
(\Theta+\sqrt{2})^2\succeq\frac{\Qbt^T\Qbt}{\NK}\succeq \frac{\gamma_-}{2\gamma_+}.\label{dmc eq}
\end{align}
\end{lemma}
\begin{proof} This result is a direct application of Theorem \ref{H1 bound} after determining minimum/maximum eigenvalues of population covariance. The cross covariance obeys $\E[\Hbb^T\Uc]=0$ due to independence. Also, for $i>1$, the truncated state vector $\hb{i}$ is statistically identical to $\h_{L-1}$ hence $\bSio{\hb{i}}\succeq \gamma_-\Iden_n$. Consequently, $\bSio{\uu{i}}=\Iden_p$, $\frac{1}{\gamma_+}\bSio{\hb{i}}\preceq\Iden_n$ for all $i$ and $\frac{\gamma_-}{\gamma_+}\Iden_n\preceq \frac{1}{\gamma_+}\bSio{\hb{i}}$ for all $i>1$. Hence, setting $\qb_i=\begin{bmatrix}\frac{1}{\sqrt{\gamma_+}}{\hb{i}}\\{\uu{i}}\end{bmatrix}$, for all $i>1$
\[
\frac{\gamma_-}{\gamma_+}\Iden_n\preceq\bSio{\qb_i}\preceq \Iden_n.
\]
Set the matrix $\bar{\Qb}=[\qb_2~\dots~\qb_{\NK}]^T$ and note that $\Qb=[\qb_1~\bar{\Qb}^T]^T$. Applying Theorem \ref{H1 bound} on $\bar{\Qb}$ and Corollary \ref{H1 corr} on $\Qb$, we find that, with the desired probability,
\[
\theta+\sqrt{3/2}\geq \frac{1}{\sqrt{\NK}}\|{\Qb}\|\geq  \frac{1}{\sqrt{\NK}}\lmn{\Qb}\geq  \frac{1}{\sqrt{\NK}}\lmn{\bar{\Qb}}\geq\sqrt{\frac{N-1}{N}}\sqrt{\frac{2\gamma_-}{3\gamma_+}}\geq 0.99\times \sqrt{\frac{2\gamma_-}{3\gamma_+}}.
\]
Setting $\Eb=\M-\Hbb$ and observing $\Qbt=\Qb+[\gamma_+^{-1/2}\Eb~0]$, the impact of the perturbation $\Eb$ can be bounded naively via $\lmn{\Qb}-\gamma_+^{-1/2}\|\Eb\|\leq \lmn{\Qbt}\leq \|\Qbt\|\leq \|\Qb\|+\gamma_+^{-1/2}\|\Eb\|$. Using the assumed bound on $\|\Eb\|$, this yields
\[
\theta+\sqrt{2}\geq \frac{1}{\sqrt{\NK}}\|\tilde{\Qb}\|\geq  \frac{1}{\sqrt{\NK}}\lmn{\tilde{\Qb}}\geq \sqrt{\frac{\gamma_-}{2\gamma_+}}.
\]
This final inequality is identical to the desired bound \eqref{dmc eq}.
\end{proof}

\begin{theorem}[Data matrix condition]\label{main cond thm} Consider the nonlinear state-equation \eqref{main rel}. Given $\gamma_+\geq \gamma_->0$, define the condition number $\rho=\frac{\gamma_+}{\gamma_-}$. For some absolute constants $c,C>0$, pick a trajectory length $N$ where 
\[
L=\lceil 1-\frac{\log{{(cn\rho)}{}}}{\log \|\A\|}\rceil\quad,\quad N_0=\lfloor \frac{N}{L}\rfloor\geq C\rho^2(n+p),
\]
and pick scaling $\mu=\frac{1}{\sqrt{\gamma_+}}$. Suppose $\|\A\|<1$, $\phi$ is $\beta$-increasing, $\ub_t\distas\Nn(0,\Iden_p)$, and Assumption \ref{ass well} holds with $\gamma_+,\gamma_-,\theta,L$. Matrix $\X=[\x_1~\dots~\x_N]^T$ of \eqref{matxxx} satisfies the following with probability $1-4N\exp(-100n)-8L\exp(-\order{N_0/\rho^2})$.
\begin{itemize}
\item Each row of $\X$ has $\ell_2$ norm at most $c_0\sqrt{p+n}$ where $c_0$ is an absolute constant.
\item $\X^T\X$ obeys the bound 
\begin{align}
(\Theta+\sqrt{2})^2\Iden_{n+p}\succeq \frac{\X^T\X}{N}\succeq \rho^{-1}\Iden_{n+p}/2.\label{upplow}
\end{align}
\end{itemize}
\end{theorem}
\begin{proof} The first statement on $\ell_2$-norm bound can be concluded from Lemma \ref{l2normbound} and holds with probability $1-2N\exp(-100(n+p))$. To show the second statement, for a fixed offset $1\leq \tao\leq L$, consider Definition \ref{subsmat} and the matrices $\Hc^{(\tao)},\Uc^{(\tao)},\Xc^{(\tao)}$. Observe that $\X$ is obtained by merging multiple sub-trajectory matrices $\{\Xc^{(\tao)}\}_{\tao=1}^L$. We will first show the advertised bound for an individual $\Xc^{(\tao)}$ by applying Lemma \ref{dmc thm} and then apply Lemma \ref{lem merge} to obtain the bound on the combined matrix $\X$.

Recall that $\NK_{\tao}$ is the length of the $\tao$th sub-trajectory i.e.~number of rows of $\Xc^{(\tao)}$. By construction $2N_0\geq \NK_{\tao}\geq N_0$ for all $1\leq \tao\leq L$. Given $1\leq \tao\leq L$ and triple $\Hbb^{(\tao)},\Hc^{(\tao)},\Uc^{(\tao)}$, set $\Qb=[\mu\Hbb^{(\tao)}~\Uc^{(\tao)}]$. Since $N_0$ is chosen to be large enough, applying Theorem \ref{dmc thm} with $\mu=1/\sqrt{\gamma_+}$ choice, and noting $\rho=\gamma_+/\gamma_-$, we find that, with probability $1-4\exp(-c_1N_0/\rho^2)$, all matrices $\M$ satisfying $\|\M-\Hbb^{(\tao)}\|\leq \sqrt{\gamma_-N_0}/10$ and $\Qbt$ as in \eqref{q mat} obeys
\begin{align}
(\Theta+\sqrt{2})^2\succeq \frac{\Qbt^T\Qbt}{N}\succeq \rho^{-1}/2.\label{qbt indiv}
\end{align}
Let us call this Event 1. To proceed, we will argue that with high probability $\|\Hc^{(\tao)}-\Hbb^{(\tao)}\|$ is small so that the bound above is applicable with $\M=\Hc^{(\tao)}$ choice; which sets $\Qbt=\Xc^{(\tao)}$ in \eqref{qbt indiv}. Applying Lemma \ref{trun lem rand}, we find that, with probability $1-2\NK_{\tao}\exp(-100n)$,
\[
\|\Hbb^{(\tao)}-\Hc^{(\tao)}\|\leq \sqrt{2N_0}\max\{\tn{\hi{i}-\hb{i}}\}\leq c_0\sqrt{2N_0}\sqrt{n\gamma_+}\|\A\|^{L-1}.
\]
Let us call this Event 2. We will show that our choice of $L$  ensures right hand side is small enough and guarantees $\|\Hbb^{(\tao)}-\Hc^{(\tao)}\|\leq \sqrt{\gamma_-N_0}/10$. Set $c=\max\{200c_0^2,1\}$. Desired claim follows by taking logarithms of upper/lower bounds and cancelling out $\sqrt{N_0}$ terms as follows
\begin{align}
 c_0\sqrt{n}\|\A\|^{L-1}\sqrt{\gamma_+}\leq \sqrt{\gamma_-}/10\sqrt{2}&\iff (L-1)\log \|\A\|+\log \sqrt{cn\rho}\leq 0\\
 &\iff -\frac{\log cn\rho}{2\log \|\A\|} \leq L-1\\
 &\impliedby L=\lceil1-\frac{\log{{(cn\rho )}{}}}{\log \|\A\|}\rceil.
\end{align}
Here we use the fact that $\log \|\A\|<0$ since $\|\A\|<1$ and $cn\rho\geq 0$. Consequently, both Event 1 and Event 2 hold with probability $1-4\exp(-c_1N_0/\rho^2)-2\NK_{\tao}\exp(-100n)$, implying \eqref{qbt indiv} holds with $\Qbt=\Xc^{(\tao)}$. Union bounding this over $1\leq\tao\leq L$, \eqref{qbt indiv} uniformly holds with $\Qbt=\Xc^{(\tao)}$ and all rows of $\X$ are $\ell_2$-bounded with probability $1-4N\exp(-100n)-8L\exp(-c_1N_0/\rho^2)$. Applying Lemma \ref{lem merge} on $(\Xc^{(\tao)})_{\tao=1}^L$, we conclude with the bound \eqref{upplow} on the merged matrix $\X$.
\end{proof}
\begin{lemma} [$\ell_2$-bound on rows]\label{l2normbound}Consider the setup of Theorem \ref{main cond thm}. With probability $1-2N\exp(-100(n+p))$, each row of $\X$ has $\ell_2$-norm at most $c\sqrt{p+n}$ for some constant $c>0$.
\end{lemma}
\begin{proof} The $t$th row of $\X$ is equal to $\x_t=[\frac{\h_t^T}{\sqrt{\gamma_+}}~{\ub_t^T}{}]^T$. Since $\tsub{\h_t-\E[\h_t]}\leq \order{\sqrt{\gamma_+}}$ and $\tsub{\ub_t}\leq \order{1}$, we have that $\tsub{\x_t-\E[\x_t]}\leq \order{1}$. Now, applying Lemma \ref{lemma sublen} on all rows $\{\x_t\}_{t=1}^N$, and using a union bound, with probability at least $1-2N\exp(-100(n+p))$, we have that $\tn{\x_t-\E[\x_t]}\leq c\sqrt{n+p}$ for all $t$. To conclude, note that $\tn{\E[\x_t]}=\tn{\E[\h_t]}/\sqrt{\gamma_+}\leq\theta\leq3\sqrt{n}$ via Assumption \ref{ass well}.
\end{proof}

\section{Proofs of Main Results}
\subsection{Proof of Lemma \ref{stcov}}
The statement follows from upper bound Lemma \ref{upp bound} and lower bound Lemma \ref{lwbnd1}.
\subsection{Proof of Theorem \ref{MAIN}}
\begin{proof} To prove this theorem, we combine Theorem \ref{main cond thm} with deterministic SGD convergence result of Theorem \ref{det conv}. Applying Theorem \ref{main cond thm}, with the desired probability, inequality \eqref{upplow} holds and for all $i$, input data satisfies the bound $\tn{\x_i}\leq \sqrt{(n+p)/(2c_0)}$ for a sufficiently small constant $c_0>0$. As the next step, we will argue that these two events imply the convergence of SGD.

Let $\bteta^{(i)},\cb^{(i)}\in\R^{n+p}$ denote the $i$th rows of $\bTeta,\Cb$ respectively. Observe that the square-loss is separable along the rows of $\Cb$ via $\tf{\bTeta-\Cb}^2=\sum_{i=1}^n\tn{\bteta^{(i)}-\cb^{(i)}}^2$. Hence, SGD updates each row $\cb^{(i)}$ via its own state equation
\[
\y_{t,i}=\phi(\li\cb^{(i)},\x_t\ri),
\]
where $\y_{t,i}$ is the $i$th entry of $\y_t$. Consequently, we can establish the convergence result for an individual row of $\Cb$. Convergence of all individual rows will imply the convergence of the overall matrix $\bTeta_{\tau}$ to the ground truth $\Cb$. Pick a row index $i$ ($1\leq i\leq n$), set $\cb=\cb^{(i)}$ and denote $i$th row of $\bTeta_\tau$ by $\bteta_\tau$. Also denote the label corresponding to $i$th row by $y_t=\y_{t,i}$. With this notation, SGD over \eqref{sgd loss} runs SGD over the $i$th row with equations $y_t=\phi(\li\cb,\x_t\ri)$ and with loss functions 
\[
\Lc(\bteta)=N^{-1}\sum_{t=1}^N\Lc_t(\bteta),~\Lc_t(\bteta)=\frac{1}{2}(y_t-\phi(\li\bteta,\x_t\ri))^2.
\]
Substituting our high-probability bounds on $\x_t$ (e.g.~\eqref{upplow}) into Theorem \ref{det conv}, we can set $B=(n+p)/(2c_0)$, $\gamma_+=(\theta+\sqrt{2})^2$, and $\gamma_-=\rho^{-1}/2$. Consequently, using the learning rate $\eta=c_0\frac{\beta^2\rho^{-1}}{(\theta+\sqrt{2})^2(n+p)}$, for all $\tau\geq 0$, the $\tau$th SGD iteration $\bteta_\tau$ obeys
\begin{align}
\E[\tn{\bteta_\tau-\cb}^2]\leq \tn{\bteta_0-\cb}^2(1-c_0\frac{\beta^4\rho^{-2}}{2(\theta+\sqrt{2})^2(n+p)})^\tau,\label{sgd ineqs}
\end{align}
where the expectation is over the random selection of SGD updates. This establishes the convergence for a particular row of $\Cb$. Summing up these inequalities \eqref{sgd ineqs} over all rows $\bteta^{(1)}_{\tau},\dots,\bteta^{(n)}_{\tau}$ (which converge to $\cb^{(1)},\dots,\cb^{(n)}$ respectively) yields the targeted bound \eqref{conv bound}.
\end{proof}

\subsection{Proofs of main results on stable systems}

\subsubsection{Proof of Theorem \ref{main thm}}
\begin{proof}Applying Lemmas \ref{upp bound} and \ref{stcov}, independent of $L$, Assumption \ref{ass well} holds with parameters
\[
\gamma_+=B_\infty^2\quad,\quad\gamma_-=\beta^2\lmn{\B}^2\quad,\quad \theta=\sqrt{6n}-\sqrt{2}\geq \sqrt{n}.
\] This yields $(\theta+\sqrt{2})^2=6n$. Hence, we can apply Theorem \ref{MAIN} with the learning rate $\eta=c_0\frac{\beta^2}{6\rho n(n+p)}$ where 
\begin{align}
\rho=\frac{B_\infty^2}{\beta^2\lmn{\B}^2}=\frac{\gamma_+}{\gamma_-},\label{rho def}
\end{align} and convergence rate $1-\frac{\beta^2\eta}{2\rho}$. To conclude with the stated result, we use the change of variable $c_0/6\rightarrow c_0$.
\end{proof}

\subsubsection{Proof of Theorem \ref{thm odd}}
\begin{proof} The proof is similar to that of Theorem \ref{main thm}. Applying Lemmas \ref{upp bound}, \ref{lem odd}, and \ref{stcov}, independent of $L$, Assumption \ref{ass well} holds with parameters
\[
\gamma_+=B_\infty^2\quad,\quad\gamma_-=\lmn{\B}^2\quad,\quad \theta=0.
\] 
Hence, we again apply Theorem \ref{MAIN} with the learning rate $\eta=c_0\frac{\beta^2}{2\rho (n+p)}$ where $\rho$ is given by \eqref{rho def}. Use the change of variable $c_0/2\rightarrow c_0$ to conclude with the stated result.
\end{proof}

\subsection{Learning unstable systems}
In a similar fashion to Section \ref{gen strat}, we provide a more general result on unstable systems that makes a parametric assumption on the statistical properties of the state vector.
\begin{assumption}[Well-behaved state vector -- single timestamp] \label{ass well2} Given timestamp $T_0>0$, there exists positive scalars $\gamma_+,\gamma_-,\theta$ and an absolute constant $C>0$ such that $\theta\leq 3\sqrt{n}$ and the following holds
\begin{align}
\gamma_+\Iden_n\succeq \bSio{\h_{T_0}}\succeq \gamma_-\Iden_n\quad\text{,}\quad\tsub{\h_{T_0}-\E[\h_{T_0}]}\leq C\sqrt{\gamma_+}\quad\text{and}\quad\tn{\E[\h_t]}\leq \theta\sqrt{\gamma_+ }.\label{assupp bound2}
\end{align}
\end{assumption}
The next theorem provides the parametrized result on unstable systems based on this assumption.
\begin{theorem} [Unstable system - general] \label{unstab sys}Suppose we are given $N$ independent trajectories $(\h^{(i)}_{t},\ub^{(i)}_{t})_{t\geq 0}$ for $1\leq i\leq N$. Sample each trajectory at time $T_0$ to obtain $N$ samples $(\y_{i},\h_{i},\ub_{i})_{i=1}^N$ where $i$th sample is
\[
(\y_{i},\h_{i},\ub_{i})=(\h^{(i)}_{T_0+1},\h^{(i)}_{T_0},\ub^{(i)}_{T_0}).
\]
Let $C,c_0>0$ be absolute constants. Suppose Assumption \ref{ass well} holds with $T_0$ and sample size satisfies $N\geq C \rho^2 (n+p)$ where $\rho=\gamma_+/\gamma_-$. Assume $\phi$ is $\beta$-increasing, zero initial state conditions, and $\ub_t\distas\Nn(0,\Iden_p)$. Set scaling to be $\mu=1/\sqrt{\gamma_+}$ and learning rate to be $\eta=c_0\frac{\beta^2}{\rho(\theta+\sqrt{2})^2(n+p)}$. Starting from $\bTeta_0$, we run SGD over the equations described in \eqref{reparam} and \eqref{sgd loss}. With probability $1-2N\exp(-100(n+p))-4\exp(-\order{\frac{N}{\rho^2}})$, all iterates satisfy
\[
\E[\tf{\Theta_{i}-\Cb}^2]\leq (1-c_0\frac{\beta^4}{2\rho^2 (\theta+\sqrt{2})^2(n+p)})^{\tau}\tf{\Theta_{0}-\Cb}^2,
\]
where the expectation is over the randomness of the SGD updates.
\end{theorem}
\begin{proof} Set $\x_i=[\gamma_+^{-1/2}\h_i^T~\ub_i^T]^T$ and $\X=[\x_1~\dots~\x_N]^T$. Since $\X$ has i.i.d.~rows, we can apply Theorem \ref{H1 bound} and Lemma \ref{lemma sublen} to find with the desired probability that
\begin{itemize}
\item Rows of $\x_i$ satisfy $\tsub{\x_i-\E[\x_i]}\leq \order{1}$ and $\E[\tn{\x_i}]\leq 3\sqrt{n}$, hence all rows of $\X$ obeys $\tn{\x_i}\leq \sqrt{(n+p)/(2c_0)}$,
\item $\X$ satisfies
\[
(\theta+\sqrt{2})^2\succeq \frac{\X^T\X}{N} \succeq \rho^{-1}/2.
\]
\end{itemize}
To proceed, using $\gamma_-=\rho^{-1}/2$, $B=(n+p)/(2c_0)$, and $\gamma_+=(\theta+\sqrt{2})^2$, we apply Theorem \ref{det conv} on the loss function \eqref{sgd loss}; which yields the desired result.
\end{proof}
\subsection{Proof of Theorem \ref{thm unstab}}

\begin{proof} The proof is a corollary of Theorem \ref{unstab sys}. We need to substitute the proper values in Assumption \ref{ass well2}. Applying Lemma \ref{upp bound}, we can substitute $\gamma_+=B_{T_0}^2$ and $\theta=\sqrt{6n}-\sqrt{2}\geq \sqrt{n}$. Next, we need to find a lower bound. Applying Lemma \ref{stcov} for $n>1$ and Lemma \ref{miso lem} for $n=1$, we can substitute $\gamma_-=\gamma_+/\rho$ with the $\rho$ definition of \eqref{uns rho}. With these, the result follows as an immediate corollary of Theorem \ref{unstab sys}.
\end{proof}

\section{Supplementary Statistical Results}

The following theorem bounds the empirical covariance of matrices with independent subgaussian rows.
\begin{theorem}\label{H1 bound} Let $\A\in\R^{n\times d}$ be a matrix with independent $\{\ab_i\}_{i=1}^n$ subgaussian rows satisfying $\tsub{\zm{\ab_i}}\leq \order{K}$ and $\bSio{\ab_i}\preceq K^2\Iden_d$ for some $K>0$ and $\tn{\E[\ab_i]}\leq \theta$. Suppose $\bSio{\ab_i}\succeq\la \Iden_d$. Suppose $n\geq \order{K^4d/\la^2}$. Then, with probability at least $1-4\exp(-cK^{-4}\la^2n)$,
\[
\theta+\sqrt{3/2}K\geq \frac{1}{\sqrt{n}}\|\A\|\geq \frac{1}{\sqrt{n}}\lmn{\A}\geq \sqrt{2\la/3}.
\]
\end{theorem}
\begin{proof} 
Let $\Eb=\E[\A],~\Aba=\A-\E[\A],~\abb_i=\zm{\ab_i}$. Observe that
\[
\tn{\A\vb}^2-\E[\tn{\A\vb}^2]=\tn{\Aba\vb}^2+2\vb^T\Aba\Eb\vb+\tn{\Eb\vb}^2-\E[\tn{\A\vb}^2]=\tn{\Aba\vb}^2-\E[\tn{\Aba\vb}^2]+2\vb^T\Aba^T\Eb\vb.
\]
Define the random process $X_{\vb}=\tn{\Aba\vb}^2$ and $Y_{\vb}=X_{\vb}-\E[X_{\vb}]$. First, we provide a deviation bound for the quantity $\sup_{\vb\in \Sc^{d-1}}|Y_{\vb}|$. To achieve this, we will utilize Talagrand's mixed tail bound and show that increments of $Y_{\vb}$ are subexpoential. Pick two unit vectors $\vb,\ub\in\R^d$. Write $\x=\ub+\vb,\y=\ub-\vb$. We have that
\[
X_{\ub}-X_{\vb}=\tn{\Aba\ub}^2-\tn{\Aba\vb}^2=\tn{\Aba(\x+\y)/2}^2-\tn{\Aba(\x-\y)/2}^2=\x^T\Aba^T\Aba\y= \sum_{i=1}^n (\abb_i^T\x)(\abb_i^T\y).
\]
Letting $\xh=\x/\tn{\x},\yh=\y/\tn{\y}$, observe that, multiplication of subgaussians $\x^T\abb_i,\y^T\abb_i$ obey
\[
\te{(\x^T\abb_i)(\y^T\abb_i)}\leq \order{\tn{\x}\tn{\y}K^2}\leq \order{K^2\tn{\y}}.
\]
Centering this subexponential variable around zero introduces a factor of $2$ when bounding subexponential norm and yields $\te{(\x^T\abb_i)(\y^T\abb_i)-\E[(\x^T\abb_i)(\y^T\abb_i)]}\leq \order{K^2\tn{\y}}$.
Now, using the fact that $Y_{\ub}-Y_{\vb}$ is sum of $n$ independent zero-mean subexponential random variables, we have the tail bound
\[
\Pro(n^{-1}|Y_{\ub}-Y_{\vb}|\geq t) \leq 2\exp(-c'n\min\{\frac{t^2}{K^4\tn{\y}^2},\frac{t}{K^2\tn{\y}}\}).
\]
 Applying Talagrand's chaining bound for mixed tail processes with distance metrics $\rho_2=\frac{K^2\tn{\cdot}}{\sqrt{n}},\rho_1=\frac{K^2\tn{\cdot}}{n}$, (Theorem $3.5$ of \cite{dirksen2013tail} or Theorem $2.2.23$ of \cite{talagrand2014gaussian}) and using the fact that for unit sphere $\Sc^{d-1}$, Talagrand's $\gamma$ functionals (see \cite{talagrand2014gaussian}) obey $\gamma_1(\Sc^{d-1}),\gamma_2^2(\Sc^{d-1})\leq \order{d}$,
 \begin{align}
  n^{-1}\sup_{\vb\in \Sc^{d-1}} |Y_{\vb}|&\leq cK^2( \sqrt{d/n}+d/n+t/\sqrt{n}),\label{bbound 1}
 \end{align}
 with probability $1-2\exp(-\min\{t^2 ,\sqrt{n}t\})$. Since $n\geq C\la^{-2}K^{4}d$ for sufficiently large $C>0$, picking $t=\frac{1}{16c}K^{-2}\la\sqrt{n}$, with probability $1-2\exp(-\order{K^{-4}\la^2n})$, we ensure that right hand side of \eqref{bbound 1} is less than $\la /8$. This leads to the following inequalities
 \begin{align}
 \frac{1}{n}\|\Aba^T\Aba-\E[\Aba^T\Aba]\|\leq \frac{\la}{8}&\implies\frac{9K^2}{8}\Iden_d\succeq\frac{1}{n}\Aba^T\Aba\succeq \frac{7\la}{8}\Iden_d.\label{my line}\\
 &\implies \frac{9}{8}K\geq \frac{1}{\sqrt{n}}\|\Aba\|\geq \lmn{\Aba}\geq  \sqrt{\frac{7}{8}\la}.\nn
 \end{align}
 Denote the size $n$ all ones vector by $\onebb_n$. Next, we define the process $Z_{\vb}=\frac{1}{\sqrt{n}}\onebb_n^T\Aba\vb$. Observe that $\Aba^T\onebb_n=\sum_{i=1}^n\abb_i\in\R^d$ is a vector satisfying $\tsub{\Aba^T\onebb_n/\sqrt{n}}\leq \order{K}$. Hence, again using $n\geq C{K^4\la^{-2}d}$ for sufficiently large $C>0$, applying Lemma \ref{lemma sublen} with $m=c_0K^{-4}\la^2{n}>d$ by picking a sufficiently small constant $c_0>1/C$, with probability at least $1-2\exp(-100c_0K^{-4}\la^2n)$
\[
\frac{1}{\sqrt{n}}\sup_{\tn{\vb}=1}|Z_{\vb}|=\frac{1}{n}\tn{\Aba^T\onebb_n}\leq \frac{1}{12}{KK^{-2}\la}\leq \frac{\sqrt{\la}}{12}.
\]
Let $\Pb=\Iden_n-\frac{1}{{n}}\onebb_n\onebb_n^T$ be the projection onto the orthogonal complement of the all ones vector. Note that $\Pb\Eb\vb=0$ as the rows of $\Eb$ are equal. With this observation, with desired probability, for any unit length $\vb$, 
\begin{align}
\tn{\A\vb}\geq \tn{\Pb\A\vb}=\tn{\Pb\Aba\vb}\geq \tn{\Aba\vb}-|Z_{\vb}|\geq \lmn{\Aba}-\sup_{\vb\in\Sc^{d-1}}|Z_{\vb}|\geq (\sqrt{7/8}-1/12)\sqrt{\la n},
\end{align}
which implies $\lmn{\A}/\sqrt{n}\geq \sqrt{2\la/3}$. For spectral norm of $\A$, we use the naive bound
\[
\frac{1}{\sqrt{n}}\|\A\|\leq \frac{1}{\sqrt{n}}(\|\Eb\|+\|\Aba\|)\leq \max_{1\leq i\leq n}\tn{\E[\ab_i]}+9K/8\leq \theta+\sqrt{3/2}K.
\]
\end{proof}
The corollary below is obtained by slightly modifying the proof above by using $ \frac{1}{n}\|\Aba^T\Aba-\E[\Aba^T\Aba]\|\leq \frac{K^2}{8}$ in line \eqref{my line} and only focusing on the spectral norm bound.
\begin{corollary}\label{H1 corr} Let $\A\in\R^{n\times d}$ be a matrix with independent $\{\ab_i\}_{i=1}^n$ subgaussian rows satisfying $\tsub{\zm{\ab_i}}\leq \order{K}$ and $\bSio{\ab_i}\preceq K^2\Iden_d$ for some $K>0$ and $\tn{\E[\ab_i]}\leq \theta$. Suppose $\bSio{\ab_i}\succeq\la \Iden_d$. Suppose $n\geq \order{K^2d}$. Then, with probability at least $1-4\exp(-cK^{-2}n)$,
\[
\theta+\sqrt{3/2}K\geq \frac{1}{\sqrt{n}}\|\A\|.
\]
\end{corollary}

The following lemma is fairly standard and is proved for the sake of completeness.
\begin{lemma} [Subgaussian vector length] \label{lemma sublen}Let $\ab\in\R^n$ be a zero-mean subgaussian vector with $\tsub{\ab}\leq L$. Then, for any $m\geq n$, there exists $C>0$ such that 
\[
\Pro(\tn{\ab}\leq CL\sqrt{m})\geq 1-2\exp(-100m).
\]
\end{lemma}
\begin{proof} We can pick a $1/2$ cover $\Cc$ of the unit $\ell_2$-sphere with size $\log |\Cc|\leq 2n$. For any $\vb\in \Cc$, subgaussianity implies, $\Pro(|\vb^T\ab|\geq t)\leq 2\exp(-\frac{ct^2}{2L^2})$. Setting $t=CL\sqrt{m}$ for sufficiently large constant $C>0$, and union bounding over all $\vb\in\Cc$, we find 
\[
\Pro(\bigcap_{\vb\in\Cc}\tn{\vb}\leq CL\sqrt{m})\geq 1-2\exp(2n-\frac{cC^2L^2m}{2L^2})\leq 1-2\exp(-100m).
\]
To conclude, let $\vb(\ab)\in\Cc$ be $\ab$'s neighbor satisfying $\tn{\vb-\frac{\ab}{\tn{\ab}}}\leq 1/2$. Hence, we have
\[
\tn{\ab}\leq \tn{(\ab-\vb(\ab))^T\ab}+\tn{\vb^T\ab}\leq \tn{\ab}/2+CL\sqrt{m}\implies \tn{\ab}\leq 2CL\sqrt{m}.
\]
To conclude, use the change of variable $C\rightarrow C/2$.
\end{proof}

\end{document}